\documentclass[journal]{IEEEtran} 
\usepackage{amsmath,amsfonts,amssymb,amsthm}

\let\labelindent\relax
\usepackage{prettyref}
\usepackage{mathrsfs}
\usepackage{graphicx}
\usepackage{wrapfig}
\usepackage{subfig}
\usepackage{bbold}
\usepackage{tabu}
\usepackage{MnSymbol}
\usepackage{multirow}
\usepackage{booktabs}
\usepackage{enumitem}
\usepackage{algorithm}
\usepackage[table]{xcolor}
\usepackage[noend]{algpseudocode}
\usepackage
[backend=bibtex,
bibstyle=ieee,
citestyle=numeric,sortcites,
mincitenames=1,
maxcitenames=2,
natbib=true,
doi=false,
isbn=false,
url=false,
eprint=false]{biblatex}
\usepackage[colorlinks,allcolors=gray,hypertexnames=true]{hyperref} 
\newcommand{\citewithauthor}[1]{\citeauthor{#1} \cite{#1}}

\newtheorem{theorem}{\TE{Theorem}}[section]
\newtheorem{lemma}[theorem]{\TE{Lemma}}

\newtheorem{assume}[theorem]{\TE{Assumption}}

\newtheorem{corollary}[theorem]{\TE{Corollary}}
\newtheorem{define}[theorem]{\TE{Definition}}
\algnewcommand{\IfThenElse}[3]{
  \State \algorithmicif\ #1\ \algorithmicthen\ #2\ \algorithmicelse\ #3}
  
\algnewcommand{\LineComment}[1]{\State \(\triangleright\) #1}
\algdef{SE}[DOWHILE]{Do}{doWhile}{\algorithmicdo}[1]{\algorithmicwhile\ #1}
\newcommand*{\colorboxed}{}
\def\colorboxed#1#{%
  \colorboxedAux{#1}%
}
\newcommand*{\colorboxedAux}[3]{%
  \begingroup
    \colorlet{cb@saved}{.}%
    \color#1{#2}%
    \boxed{%
      \color{cb@saved}%
      #3%
    }%
  \endgroup
}

\newrefformat{fig}{Figure~\ref{#1}}
\newrefformat{par}{Section~\ref{#1}}
\newrefformat{appen}{Appendix~\ref{#1}}
\newrefformat{sec}{Section~\ref{#1}}
\newrefformat{sub}{Section~\ref{#1}}
\newrefformat{table}{Table~\ref{#1}}
\newrefformat{ass}{Assumption~\ref{#1}}
\newrefformat{alg}{Algorithm~\ref{#1}}
\newrefformat{def}{Definition~\ref{#1}}
\newrefformat{thm}{Theorem~\ref{#1}}
\newrefformat{cor}{Corollary~\ref{#1}}
\newrefformat{lem}{Lemma~\ref{#1}}
\newrefformat{step}{Step~\ref{#1}}
\newrefformat{ln}{Line~\ref{#1}}
\newrefformat{rem}{Remark~\ref{#1}}
\newrefformat{eq}{Equation~\eqref{#1}}
\newrefformat{pb}{Problem~\ref{#1}}
\newrefformat{it}{Item~\ref{#1}}
\newrefformat{te}{Term~\ref{#1}}
\def\Eqref Eq:#1:{\eqref{eq:#1}}
\newrefformat{Eq}{Equation~\Eqref#1:}

\newcommand{\TE}[1]{\textbf{#1}}

\newcommand{\FPP}[2]{\frac{\partial{#1}}{\partial{#2}}}
\newcommand{\FPPR}[2]{{\partial{#1}}/{\partial{#2}}}
\newcommand{\FPPT}[2]{\frac{\partial^2{#1}}{\partial{#2}^2}}

\newcommand{\FDD}[2]{\frac{d{#1}}{d{#2}}}

\newcommand{\THREE}[3]{\left(\setlength{\arraycolsep}{1pt}\begin{array}{ccc}{#1}, & {#2}, & {#3}\end{array}\right)}

\newcommand{\dist}{\text{dist}}
\newcommand{\adj}{\mathcal{M}}

\newcommand{\fmax}[1]{\underset{#1}{\max}}

\newcommand{\argmin}[1]{\underset{#1}{\text{argmin}}\;}

\newcommand{\ST}{\text{s.t.}\quad}

\usepackage{xcolor}
\definecolor{Blue} {rgb}{0.4, 0.4, 0.8}
\definecolor{Red}  {rgb}{0.8, 0.2, 0.2}
\definecolor{Green}{rgb}{0.2, 0.8, 0.2}
\newcommand{\revised}[1]{\textcolor{black}{#1}}

\usepackage{comment}
\makeatletter
\IEEEtriggercmd{\reset@font\normalfont\fontsize{5.0pt}{5.0pt}\selectfont}
\makeatother
\IEEEtriggeratref{1}

\makeatletter
\newcommand\fs@ruled@notop{\def\@fs@cfont{\bfseries}\let\@fs@capt\floatc@ruled
  \def\@fs@pre{}%
  \def\@fs@post{\kern2pt\hrule\relax}%
  \def\@fs@mid{\kern2pt\hrule\kern2pt}%
  \let\@fs@iftopcapt\iftrue}
\renewcommand\fst@algorithm{\fs@ruled@notop}
\makeatother

\newif\ifsupp
\supptrue
\title{\large\bf Provably Feasible Semi-Infinite Program Under Collision Constraints via Subdivision}
\author{Duo Zhang$^1$, Chen Liang$^2$, Xifeng Gao$^1$, Kui Wu$^1$, and Zherong Pan$^{1\dagger}$  \\
\thanks{$^\dagger$ indicates corresponding author. $^1$Lightspeed Studios, Tencent. \{zduozhang, xifgao, kwwu, zrpan\}@global.tencent.com. $^2$The Department of Computer Science, Zhejiang university.}}
        
\begin{document}
\maketitle
\allowdisplaybreaks
\thispagestyle{empty}
\pagestyle{empty}

\begin{abstract}
We present a semi-infinite program (SIP) solver for trajectory optimizations of general articulated robots. These problems are more challenging than standard Nonlinear Program (NLP) by involving an infinite number of non-convex, collision constraints. Prior SIP solvers based on constraint sampling cannot guarantee the satisfaction of all constraints. Instead, our method uses a conservative bound on articulated body motions to ensure the solution feasibility throughout the optimization procedure. We further use subdivision to adaptively reduce the error in conservative motion estimation. Combined, we prove that our SIP solver guarantees feasibility while approaching the optimal solution of SIP problems up to arbitrary user-provided precision. We demonstrate our method towards several trajectory optimization problems in simulation, including industrial robot arms and UAVs. The results demonstrate that our approach generates collision-free locally optimal trajectories within a couple of minutes.
\end{abstract}
\section{Introduction}
This paper deals with trajectory generation for articulated robots, which is a fundamental problem in robotic motion planning. Among other requirements, providing strict collision-free guarantees is crucial to a reliable algorithm, i.e., the robot body should be bounded away from static and dynamic obstacles by a safe distance at any time instance. In addition to feasibility, modern planning algorithms such as~\cite{doi:10.1177/0278364914528132} further seek (local) optimality, i.e., finding trajectories that correspond to the minimizers of user-specified cost functions. Typical cost functions would account for smoothness~\cite{kyriakopoulos1988minimum}, energy efficacy~\cite{hansen2012enhanced}, and time-optimization~\cite{kunz2012time}. Despite decades of research, achieving simultaneous feasibility and optimality remains a challenging problem.

Several categories of techniques have attempted to generate feasible and optimal trajectories. The most widely recognized sampling-based motion planners~\cite{lavalle1998rapidly} and their optimal variants~\cite{doi:10.1177/0278364911406761} progressively construct a tree in the robot configuration space and then use low-level collision checker to ensure collision-free along each edge of the tree. The optimal trajectory restricted to the tree can asymptotically approach the global optima. However, most of the low-level collision checkers are based on discrete-time sampling~\cite{pan2012fcl} and cannot ensure collision-free during continuous motion. Optimal sampling-based methods are a kind of zeroth-order optimization algorithm that does not require gradient information to guide trajectory search. On the downside, the complexity of finding globally optimal trajectories is exponential in the dimension of configuration spaces~\cite{doi:10.1177/0278364917714338}.

In parallel, first- and second-order trajectory optimization algorithms~\cite{schulman2013finding} have been proposed to utilize derivative information to bring the trajectory towards a local optima with a polynomial complexity in the configuration space dimension. Based on the well-developed off-the-shelf NLP solvers such as~\cite{biegler2009large}, trajectory optimization has been adopted to solve complex high-dimensional planning problems. Despite the efficacy of high-order techniques, however, dealing with collision constraints becomes a major challenge as the number of constraints is infinite, leading to SIP problems~\cite{doi:10.1177/0278364920983353}. \revised{Existing trajectory optimizers for articulated robots are based on the exchange method~\cite{lopez2007semi}, i.e., sampling the constraints at discrete time instances. Similar to the discrete-time sampling used in the collision checkers, the exchange method can miss infeasible constraints and sacrifice the feasibility guarantee. In their latest works, \citewithauthor{marcucci2022motion} propose an alternative approach that first identifies feasible convex subsets of the configuration space, and then searches for globally optimal trajectories restricted to these subsets. While their method can provide feasibility and optimality guarantees, expensive precomputations are required to identify the feasible subsets~\cite{amice2022finding}.}

We propose a novel SIP solver with guaranteed feasibility. Our method is based on the discretization-based SIP solver~\cite{lopez2007semi}. We divide a robot trajectory into intervals and use conservative motion bound to ensure the collision-free property during each interval. These intervals further introduce barrier penalty functions, which guide the optimizer to stay inside the feasible domain and approach the optimal solution of SIP problems up to arbitrary user-specified precision. The key components of our method involve: 1) a motion bound that conservatively estimates the range of motion of a point on the robot over a finite time interval; 2) a safe line-search algorithm that prevents the intersection between the motion bound and obstacles; 3) a motion subdivision scheme that recursively reduces the error of conservative motion estimation. By carefully designing the motion bound and line-search algorithm, we prove that our solver converges within finitely many iterations to a collision-free, nearly locally optimal trajectory, under the mild assumption of Lipschitz motion continuity. We have also evaluated our method on a row of examples, including industrial robot arms reaching targets through complex environments and multi-UAV trajectory generation with simultaneous rotation and translation. Our method can generate safe trajectories within a couple of minutes on a single desktop machine.
\section{Related Work}
We review related works in optimality and feasibility of motion planning, SIP and its application in robotics, and various safety certifications.

\subsection{Optimality and Feasibility of Motion Planning}
We highlight three milestones in the development of motion planning frameworks: the trajectory optimization approach~\cite{betts1998survey}, the rapid-exploring random tree (RRT)~\cite{lavalle1998rapidly}, and its optimal variant (RRT-Star)~\cite{doi:10.1177/0278364911406761}. Although the performance of these frameworks largely depends on the concrete choices of algorithmic components, a major qualitative difference lies in their optimality and complexity. Trajectory optimization ensures the returned trajectory is optimal in a local basin of attraction; RRT returns an arbitrary feasible trajectory without any optimality guarantee; while RRT-Star provides an asymptotic global optimality guarantee. With the stronger guarantee in terms of optimality comes significantly higher complexity in the dimension of configuration spaces. Based on off-the-shelf NLP solvers such as~\cite{wright1997primal}, the complexity of trajectory optimization is polynomial. The complexity of RRT relies on the visibility property~\cite{hsu2007probabilistic}, which is not directly related to the dimension. Unfortunately, the complexity of optimal sampling-based motion planning is exponential~\cite{doi:10.1177/0278364917714338}, which is unsurprising considering the NP-hardness of general non-convex optimization. As a result, nearly global optimality can only be expected in low-dimensional problems, while local optimality is preferred in practical, high-dimensional planning problems. Although there have been considerable efforts in reducing the cost of RRT-Star, including the use of bidirectional exploration~\cite{jordan2013optimal}, branch-and-bound~\cite{karaman2011anytime}, informed-RRT-Star~\cite{gammell2014informed}, and lazy collision checkers~\cite{hauser2015lazy}, its worst-case complexity cannot be shaken.

In addition to optimality, providing a strict feasibility guarantee poses a major challenge for any of the aforementioned frameworks. For trajectory optimization, collision-constraints are formulated as differentiable hard constraints in NLP. However, since there are the infinite number of constraints, practical formulations~\cite{park2012itomp,schulman2013finding,zucker2013chomp} need to sample constraints at discrete time instances, which violate the feasibility guarantee. Even worse, many off-the-shelf NLP solvers~\cite{biegler2009large} can accept infeasible solutions and then use gradient information to guide the solutions back to the feasible domain, which is not guaranteed to succeed, especially when either the robot or the environment contains geometrically thin objects. An exception is the feasible SQP algorithm~\cite{Tits2009} that ensures iteration-wise feasibility, but this algorithm is not well-studied in the robotic community. On the other hand, RRT, RRT-Star, and their variants require a low-level collision checker to prune non-collision-free trajectory segments. The widely used discrete-time collision checker~\cite{pan2012fcl} again requires temporal sampling and violates the feasibility guarantee. There exists exact continuous-time collision checkers~\cite{choi2006continuous,choi2008continuous,brochu2012efficient}, but they make strong assumptions that robot links are undergoing linear or affine motions, which are only valid for point, rigid, or car-like robots. \revised{For more general articulated robot motions, inexact continuous-time collision checkers~\cite{schwarzer2004exact,10.1145/1276377.1276396,pan2011fast} have been proposed that provide motion upper bounds, but such bounds can be overly conservative and result in false negatives. In comparison, our trajectory optimization method also relies heavily on motion upper bounds, but we use recursive subdivision to adaptively tighten the motion bounds and provide both local optimality and feasibility guarantee for general configuration spaces.}

\subsection{SIP and Applications in Robotics}
SIP models mathematical programs involving a finite number of decision variables but an infinite number of constraints. SIPs frequently arise in robotic applications for modeling constraints on motion safety~\cite{doi:10.1177/0278364920983353,amice2022finding}, controllability and stability~\cite{majumdar2013control,clark2021verification}, reachability~\cite{majumdar2017funnel}, and pervasive contact realizability~\cite{zhang2021semi}. The key challenge of solving SIP lies in the reduction of the infinite constraint set to a computable finite set. To the best of our knowledge, a generally equivalent infinite-to-finite reduction is unavailable, except for some special cases~\cite{deits2015efficient,parrilo2000structured}. Therefore, general-purpose SIP solvers~\cite{lopez2007semi} rely on approximate infinite-to-finite reductions that transform SIP to a conventional NLP, which is then solved iteratively as a sub-problem. Two representative methods of this kind are the exchange and discretization methods. These methods sample the constraint index set to approximately reduce SIP to NLP. In terms of our collision constraints, this treatment resembles the discrete-time collision checker used in sampling-based motion planners. Unfortunately, even starting from a feasible initial point, general-purpose SIP solvers cannot guarantee the feasibility of solutions, which is an inherited shortcoming of the underlying NLP solver. Instead, we propose a feasible discretization method for solving the special SIP under collision constraints with a feasibility guarantee. Our method is inspired by the exact penalty formulation~\cite{pietrzykowski1969exact,conn1987exact}, which reduces the SIP to a conventional NLP by integrating over the constraint indices. The exact penalty method can be considered as a third method for infinite-to-finite reduction, but the integral in such penalty function is generally intractable to compute. Our key idea is to approximate such integrals by subdivision without compromising the theoretical guarantees.

\subsection{Planning Under Safety Certificates}
Our discussion to this point focuses on general algorithms applicable to arbitrary configuration spaces, where a feasibility guarantee is difficult to establish. But exceptions exist for several special cases or under additional assumptions. Assuming a continuous-time dynamic system, for example, the Control Barrier Function (CBF)~\cite{borrmann2015control,ames2019control} designs a controller to steer a robot while satisfying given constraints, but CBF is only concerned about the feasibility and cannot guarantee the steered robot trajectory is optimal. By approximating the robot as a point or a ball, the robot trajectory becomes a high-order spline, and tight motion bound can be derived to ensure safety. This approach is widely adopted for (multi-)UAV trajectory generation~\cite{deits2015efficient,liu2017search}, but it cannot be extended to more complex robot kinematics. Most recently, a novel formulation has been proposed by \citewithauthor{amice2022finding} to certify the feasibility of a positive-measure subset of the configuration space of arbitrary articulated robots. Their method relies on reformulation that transforms the collision constraint to a conditional polynomial positivity problem, which can be further combined with mixed-integer convex programming, as done in~\cite{marcucci2022motion}, to ensure feasibility. Compared to all these techniques, our feasibility guarantee is based on a much weaker assumption of Lipschitz motion bound, and we do not require any precomputation to establish the safety certificate.
\section{Problem Statement}
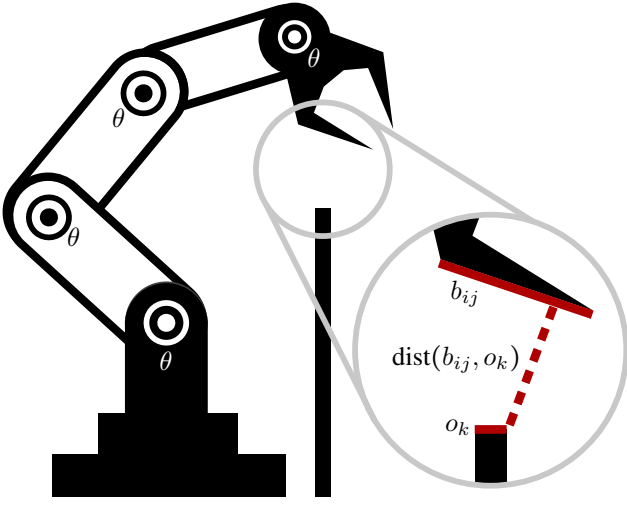
\begin{figure}[t]
\centering
\includegraphics[width=.95\linewidth]{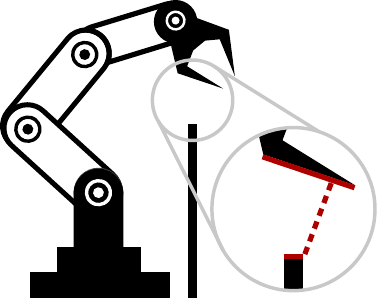}
\put(-124,163){\textcolor{white}{$\theta$}}
\put(-180,48){\textcolor{white}{$\theta$}}
\put(-198,140){$\theta$}
\put(-215,95){$\theta$}
\put(-92,50){$\text{dist}(b_{ij},o_k)$}
\put(-70,75){$b_{ij}$}
\put(-72,25){$o_k$}
\caption{\label{fig:illus} We consider a moving articulated robot arm, where the volume occupied by the $i$th link is denoted as $b_i$. Each $b_i$ admits a finite decomposition $b_i=\bigcup_j b_{ij}$ and each $b_{ij}$ is a simple shape, e.g. the red edge. $b_{ij}$ is a function of the time $t$ and trajectory parameters $\theta$, denoted as $b_{ij}(t,\theta)$. $\theta$ could be the control points of B\'ezier curves in the configuration space. Similarly, we can decompose the obstacle $o=\bigcup_k o_k$ where $o_k$ is the short red edge. We introduce log barrier energy bounding the distance $\text{dist}(b_{ij},o_k)$ (dashed line) away from a safety distance $d_0$.}
\end{figure}
In this section, we provide a general formulation for collision-constrained trajectory generation problems. Throughout the paper, we will use subscripts to index geometric entities or functions, but we choose not to indicate the total number of indices to keep the paper succinct, e.g., we denote $\sum_i$ as a summation over all indices $i$. We consider an open-loop articulated robot as composed of several rigid bodies. The $i$th rigid body occupies a finite volume in the global frame, denoted as $b_i\subset\mathbb{R}^3$. Without ambiguity, we refer to the rigid body and its volume interchangeably. We further denote $b_i^0\subset\mathbb{R}^3$ as the volume of $i$th body in its local frame. We further assume there is a set of static obstacles taking up another volume $o\subset\mathbb{R}^3$. By the articulated body kinematics, we can compute $b_i$ from $b_i^0$ via the rigid transform: $b_i=M_ib_i^0$ where we define $M_ib_i^0=\{M_ix|x\in b_i^0\}$. When a robot moves, $M_i$ and thus $b_i$ are time-dependent functions, denoted as $M_i(t,\theta)$ and $b_i(t,\theta)$, respectively. Here $t\in[0,T]$ is the time parameter and the trajectory is parameterized by a set of decision variables, denoted as $\theta$. The problem of collision-constrained trajectory generation aims at minimizing a twice-differentiable cost function $\mathcal{O}(\theta)$, such that each rigid body $b_i$ is bounded away from $o$ by a user-specified safe distance denoted as $d_0$ at any $t\in[0,T]$. Formally, this is defined as:
\begin{equation}
\begin{aligned}
\label{eq:prob}
\argmin{\theta}&\mathcal{O}(\theta)\\
\ST&\dist(b_i(t,\theta),o)\geq d_0\quad\forall i\land t\in[0,T],
\end{aligned}
\end{equation}
where $\dist(\bullet)$ is the shortest Euclidean distance between two sets. Under the very mild assumption of being twice-differentiable, the cost function $\mathcal{O}(\theta)$ can encode various user requirements for a ``good'' trajectory, i.e., the closedness between an end-effector and a target position, or the smoothness of motion. This is a SIP due to the infinitely many constraints, one corresponding to each time instance. Further, the SIP is non-smooth as the distance function between two general sets is non-differentiable. \prettyref{eq:prob} is a general definition incorporating various geometric representations of the robot and obstacles as illustrated in~\prettyref{fig:illus}.

\subsection{Smooth Approximation}
Although the main idea of this work is a discretization method for solving \prettyref{eq:prob}, most existing SIP solvers already adopt the idea of discretization for spatial representation of a rigid body $b_i$ to deal with non-smoothness of the function $\dist(\bullet)$. By spatial discretization, we assume that $b_i$ endows a finite decomposition $b_i=\bigcup_j b_{ij}$ where $b_{ij}$ is the $j$th subset of $b_i$ in world frame. Similarly, we can finitely decompose $o$ as $o=\bigcup_k o_k$ where $o_k$ is the $k$th subset of environmental obstacles $o$ in the world frame. If the distance function $\dist(b_{ij},o_k)$ between a pair of subsets is differentiable, then we can reduce the non-smooth SIP \prettyref{eq:prob} to the following smooth SIP:
\begin{equation}
\begin{aligned}
\label{eq:probSmooth}
\argmin{\theta}&\mathcal{O}(\theta)\\
\ST&\dist(b_{ij}(t,\theta),o_k)\geq d_0\quad\forall i,j,k\land t\in[0,T].
\end{aligned}
\end{equation}
In summary, spatial discretization is based on the following assumption:
\begin{assume}
\label{ass:Spatial}
Each $b_i$ and $o$ endows a finite decomposition denoted as $b_i=\bigcup_j b_{ij}$ and $o=\bigcup_k o_k$ such that $\dist(b_{ij},o_k)$ is sufficiently smooth for any $\left<i,j,k\right>$.
\end{assume}
\revised{\prettyref{ass:Spatial} holds for almost all computational representations of robot links. For example, common spatial discretization methods include point cloud, convex hull, and triangle mesh. In the case of the point cloud, each $b_{ij}$ or $o_k$ is a point, and $\dist(\bullet)$ is the differentiable pointwise distance. In the case of the convex hull, $\dist(\bullet)$ is the distance between a pair of convex hulls, which is non-differentiable in its exact form, but can be made sufficiently smooth by slightly bulging each convex hull to make them strictly convex~\cite{6710113}. In the case of the triangle mesh, it has been shown that the distance between a pair of triangles can be reduced to two sub-cases: 1) the distance between a point and a triangle and 2) the distance between a pair of edges, see~\cite{choi2006continuous}, both of which are special cases of the convex hull. Although spatial discretization can generate many more distance constraints, only a few constraints in close proximity to each other need to be activated and forwarded to the SIP solver for consideration, and these potentially active constraints can be efficiently identified using a spatial acceleration data structure~\cite{10.1145/1576246.1531393}. Despite these spatial discretizations, however, the total number of constraints is still infinite in the temporal domain.}

\subsection{The Exchange Method}
The exchange method is a classical algorithm for solving general SIP, which reduces SIP to a series of NLP by sampling constraints both spatially and temporally. Specifically, the algorithm maintains an instance set $\mathcal{I}$ that contains a finite set of $\left<i,j,k,t\right>$ tuples and reduces \prettyref{eq:prob} to the following NLP:
\begin{equation}
\begin{aligned}
\label{eq:probNLP}
\argmin{\theta}&\mathcal{O}(\theta)\\
\ST&\dist(b_{ij}(t,\theta),o_k)\geq d_0\quad\forall\left<i,j,k,t\right>\in\mathcal{I}.
\end{aligned}
\end{equation}
The algorithm approaches the solution of \prettyref{eq:prob} by alternating between solving \prettyref{eq:probNLP} and updating $\mathcal{I}$. The success of the exchange method relies on a constraint selection oracle for updating $\mathcal{I}$. Although several heuristic oracles are empirically effective, we are unaware of any exchange method that can guarantee the satisfaction of semi-infinite constraints. Indeed, most exchange methods insert new $\left<i,j,k,t\right>$ pairs into $\mathcal{I}$ when the constraint is already violated, i.e., $\dist(b_{ij}(t,\theta),o_k)<d_0$, and relies on the underlying NLP solver to pull the solution back onto the constrained manifold, where the feasibility guarantee is lost. 
\begin{figure}[t]
\centering
\includegraphics[width=.95\linewidth]{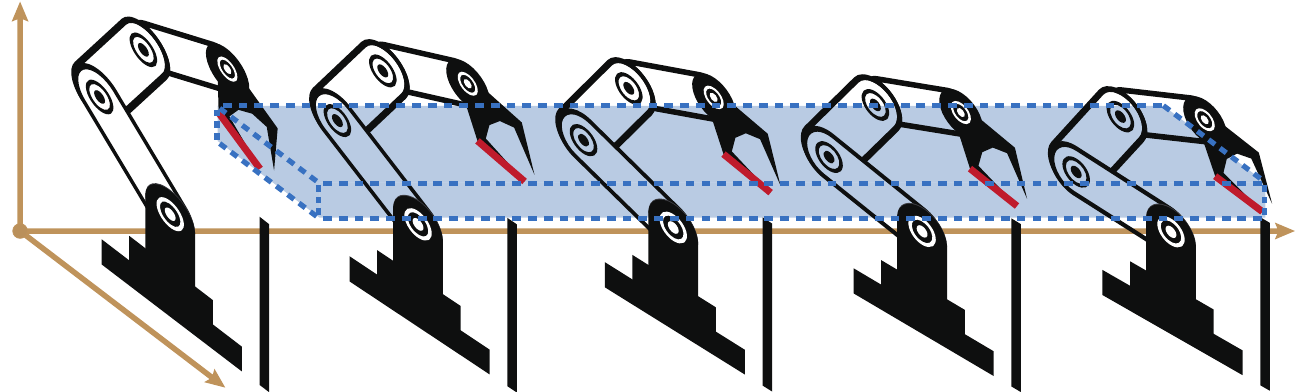}
\put(-200,-5){$x$}
\put(-240,75){$y$}
\put(-5,20){$t$}
\caption{\label{fig:illusAll} When the red edge illustrated in~\prettyref{fig:illus} is tracing out a temporal trajectory, we use a spatial-temporal motion bound (blue) to estimate its range and guarantee safety.
}
\end{figure}
\section{Subdivision-Based SIP Solver}
We propose a novel subdivision-based SIP solver inspired by the discretization method~\cite{lopez2007semi}. Unlike the exchange method that selects the instance set $\mathcal{I}$ using an oracle algorithm, the discretization method uniformly subdivides the index set into finite intervals and chooses a surrogate index from each interval to form the instance set $\mathcal{I}$, reducing the original problem into an NLP. As a key point of departure from the conventional infeasible discretization method, however, we design the surrogate constraint in \prettyref{sec:MotionBound} such that its feasible domain is a strict subset of the true feasible domain of~\prettyref{eq:prob}. We then show in~\prettyref{sec:Penalty} and~\prettyref{sec:PrimalIPM} that, by using the feasible interior point method such as~\cite[Chapter~4.1]{bertsekas1997nonlinear} to solve the NLP, our algorithm is guaranteed to generate iterations satisfying all the surrogate constraints. Since our surrogate constraint can limit the solution to an overly conservative subset, in \prettyref{sec:Subdivision}, we introduce a subdivision method to adaptively adjust the conservative subset and approach the original feasible domain.

\subsection{\label{sec:MotionBound}Surrogate Constraint}
We consider the following infinite spatial-temporal subset of constraints:
\begin{align}
\label{eq:subset}
\dist(b_{ij}(t,\theta),o_k)\geq d_0\quad\forall t\in[T_0,T_1]\subseteq[0,T],
\end{align}
where $b_{ij}$ and $o_k$ are two spatial subsets and $[T_0,T_1]\subseteq[0,T]$ is a temporal subset. Our surrogate constraint replaces the entire time interval with a single time instance. A natural choice is to use the following midpoint constraint:
\begin{align}
\label{eq:surrogate}
\dist\left(b_{ij}\left(\frac{T_0+T_1}{2},\theta\right),o_k\right)\geq d_0,
\end{align}
which is differentiable by~\prettyref{ass:Spatial}. Unfortunately, the domain specified by~\prettyref{eq:surrogate} is larger than that of~\prettyref{eq:subset}, violating our feasibility requirement. We remedy this problem by upper-bounding the feasibility error due to the use of our surrogate. A linear upper bound can be established by taking the following mild assumption:
\begin{assume}
\label{ass:Bounded}
The feasible domain of $t$ and $\theta$ is bounded.
\end{assume}
\begin{lemma}
\label{lem:bound}
Under \prettyref{ass:Spatial}, \ref{ass:Bounded}, there exists a constant $L_1$ such that:
\begin{align*}
|\dist(b_{ij}(t_1,\theta),o_k)-\dist(b_{ij}(t_2,\theta),o_k)|\leq L_1|t_1-t_2|.
\end{align*}
\end{lemma}
\begin{proof}
A differentiable function in a bounded domain is also Lipschitz continuous so that we can define $L_1$ as the Lipschitz constant.
\end{proof}
The above result implies that the feasibility error of the midpoint surrogate constraint is upper bounded by $L_1(T_1-T_0)/2$. Further, the feasible domain is specified by the following more strict constraint:
\begin{align*}
\dist(b_{ij}(t,\theta),o_k)\geq d_0+L_1|T_1-T_0|/2,
\end{align*}
is a subset of the true feasible domain. However, such a subset can be too restrictive and oftentimes lead to an empty feasible domain. Instead, our method only uses~\prettyref{lem:bound} as an additional safety check as illustrated in~\prettyref{fig:illusAll}, while the underlying optimizer deals with the standard constraint~\prettyref{eq:surrogate}. Note that the bound in~\prettyref{lem:bound} is not tight and there are many sophisticated upper bounds that converge superlinearly, of which a well-studied method is the Taylor model~\cite{10.1145/1276377.1276396}. Although we recommend using the Taylor model in the implementation of our method, our theoretical results merely require a linear upper bound.

\subsection{\label{sec:Penalty}Barrier Penalty Function}
To ensure our algorithm generates feasible iterations, we have to solve the NLP using a feasible interior-point method such as~\cite[Chapter~4.1]{bertsekas1997nonlinear}. These algorithms turn each inequality collision constraint into the following penalty function:
\begin{align*}
\mathcal{P}_{ijk}(t,\theta)\triangleq\mathcal{P}\left(\dist\left(b_{ij}\left(t,\theta\right),o_k\right)-d_0\right),
\end{align*}
where $\mathcal{P}$ is a sufficiently smooth, monotonically decreasing function defined on $(0,\infty)$ such that $\lim_{x\to0}\mathcal{P}(x)=\infty$ and $\lim_{x\to\infty}\mathcal{P}(x)=0$. In order to handle SIP problems, we need the following additional assumption to hold for $\mathcal{P}$:
\begin{assume}
\label{ass:Barrier}
The barrier function $\mathcal{P}$ satisfies: 
\begin{align*}
\lim_{x\to0}x\mathcal{P}(x)=\infty.
\end{align*}
\end{assume}
The most conventional penalty function is the log-barrier function $\mathcal{P}(x)=-\log(x)$, but this function violates \prettyref{ass:Barrier}. By direct verification, one could see that a valid penalty function is $\mathcal{P}(x)=-\log(x)/x$. In \cite{10.1145/1576246.1531393}, authors showed that a locally supported $\mathcal{P}$ is desirable for a spatial acceleration data structure to efficiently prune inactive constraints, for which we propose the following function:
\begin{align*}
\vcenter{\hbox{\includegraphics[width=.55\columnwidth]{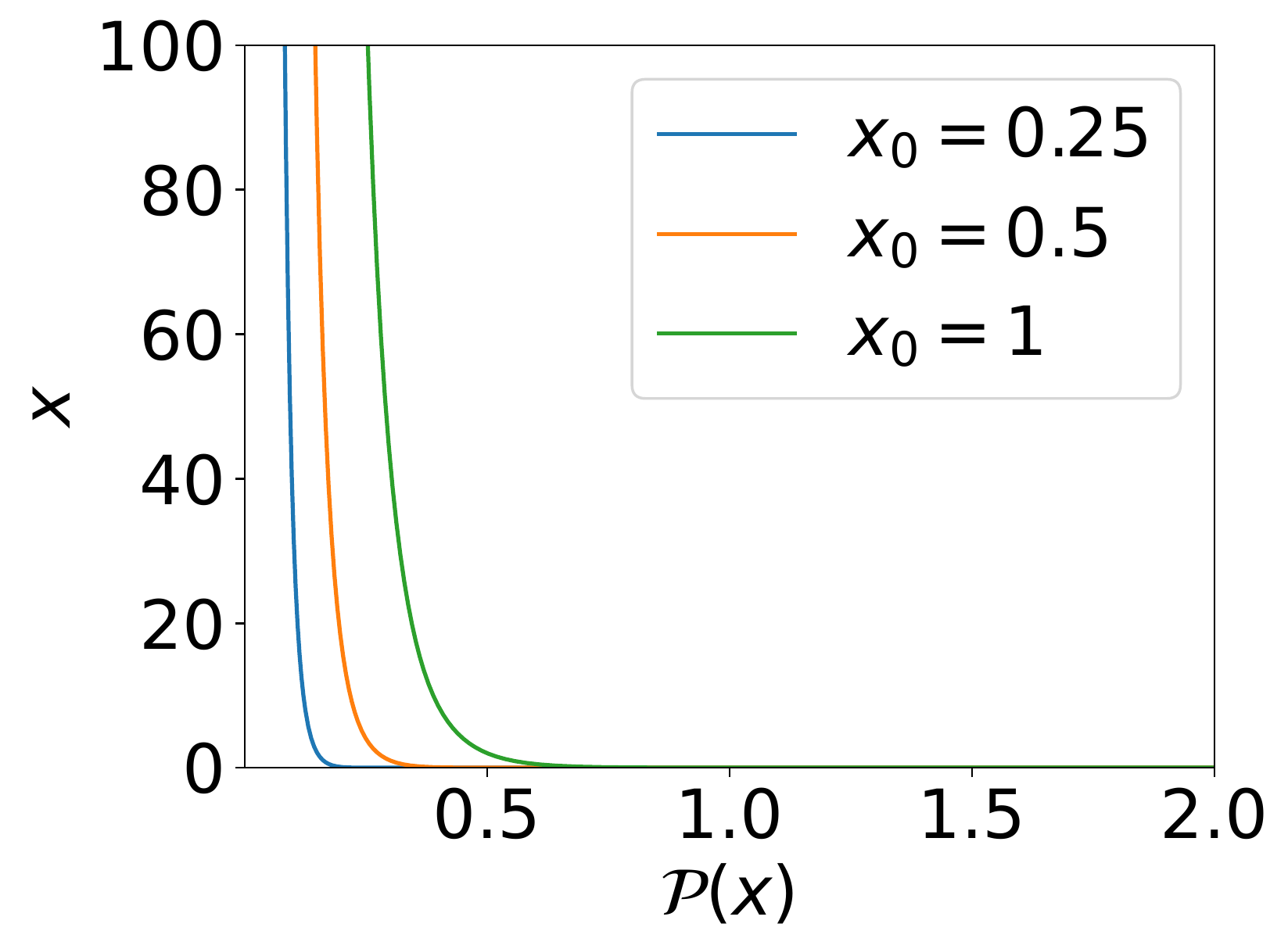}}}
\mathcal{P}(x)=\begin{cases}
\frac{(x_0-x)^3}{x^4}\quad&x\leq x_0\\
0\quad&x>x_0,
\end{cases}
\end{align*}
which is twice differentiable and locally supported within $(0,x_0]$ with $x_0$ being a small positive constant. 
The intuition behind \prettyref{ass:Barrier} lies in the integral reformulation of semi-infinite constraints. Indeed, we can transform the infinite constraints into a finite form by integrating the penalty function over semi-infinite variables, giving the following finite integral penalty function, denoted as $\bar{\mathcal{P}}$:
\begin{align}
\label{eq:integral}
\bar{\mathcal{P}}_{ijk}(T_0,T_1,\theta)\triangleq\int_{T_0}^{T_1}\mathcal{P}_{ijk}(t,\theta)dt.
\end{align}
The above integral penalty function has been considered in~\cite{pietrzykowski1969exact,conn1987exact,schattler1996interior} to solve SIP. However, their proposed algorithms are only applicable for special forms of constraints, where the integral in~\prettyref{eq:integral} has a closed-form expression. Unfortunately, such an integral in our problem does not have a closed-form solution. Instead, we propose to approximate the integral via spatial-temporal discretization. We will show that the error in our discrete approximation is controllable, which is crucial to the convergence of our proposed solver.
In order for the penalty function to guarantee feasibility, $\bar{\mathcal{P}}_{ijk}$ must tend to infinity when:
\begin{align}
\label{eq:infty}
\exists t\in[T_0,T_1]\quad\dist\left(b_{ij}\left(t,\theta\right),o_k\right)\to d_0.
\end{align}
\begin{wrapfigure}{r}{0.2\textwidth}
\vspace{-5px}
\scalebox{0.8}{
\includegraphics[width=.2\textwidth]{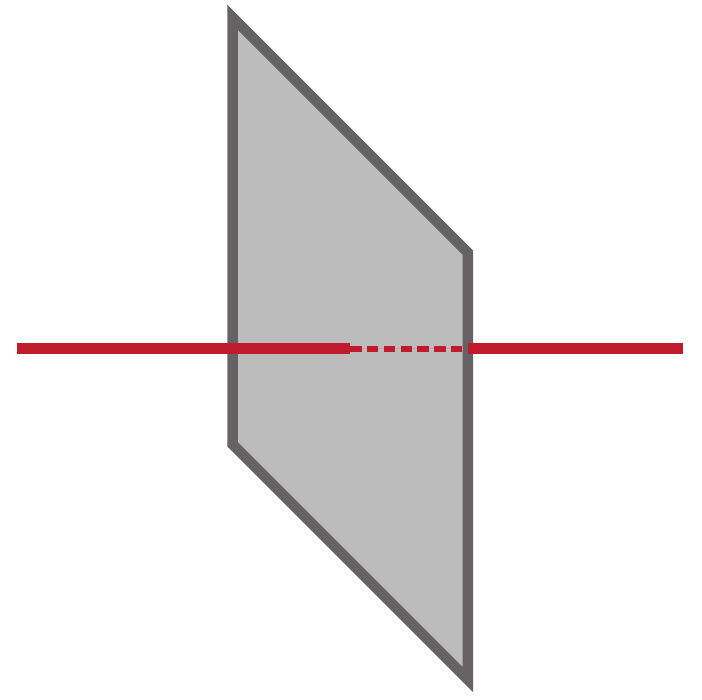}
\put(-100,40){\small$b_{ij}(0)$}
\put(- 25,40){\small$b_{ij}(1)$}
\put(-62,100){\small$o$}}
\vspace{-5px}
\end{wrapfigure}
However, the log-barrier function does not satisfy this property. As illustrated in the inset, suppose there is a straight line trajectory $b_{ij}(t)=\THREE{t}{0}{0}$ along the positive X-axis, $o$ is the YZ-plane that intersects the X-axis at $\THREE{1/2}{0}{0}$, $T=1$ and $d_0=0$, then $\bar{\mathcal{P}}$ takes the following finite value:
\begin{align*}
\bar{\mathcal{P}}_{ijk}(0,1)\triangleq\int_0^1-\log\left(\left|t-\frac{1}{2}\right|\right)dt= \log(2)+1<\infty.
\end{align*}
Instead, our \prettyref{ass:Barrier} could ensure the well-definedness of $\bar{\mathcal{P}}$ as shown in the following lemma:
\begin{lemma}
\label{lem:infty}
Suppose \prettyref{ass:Spatial}, \ref{ass:Bounded}, \ref{ass:Barrier}, and \prettyref{eq:infty} holds, then: $\bar{\mathcal{P}}_{ijk}\to\infty$.
\end{lemma}
\begin{proof}
By \prettyref{ass:Spatial}, we have the finite decomposition and $\bar{\mathcal{P}}_{ijk}$ is well-defined. Without loss of generality, we assume $t\in(T_0,T_1)$ so we can pick a positive $\epsilon_1$ such that $[t-\epsilon_1,t+\epsilon_1]\subseteq[T_0,T_1]$. For any $t'\in[t-\epsilon_1,t+\epsilon_1]$, by the boundedness of $t'$ and \prettyref{lem:bound}, we have:
\begin{align*}
\dist\left(b_{ij}\left(t',\theta\right),o_k\right)\leq 
\dist\left(b_{ij}\left(t,\theta\right),o_k\right)+L_1|t'-t|.
\end{align*}
Putting things together, we have:
\begin{align*}
&\bar{\mathcal{P}}_{ijk}(T_0,T_1,\theta)\geq\bar{\mathcal{P}}_{ijk}(t-\epsilon_1,t+\epsilon_1,\theta)\\
\geq&\epsilon_1\mathcal{P}(\dist\left(b_{ij}\left(t,\theta\right),o_k\right)-d_0+L_1\epsilon_1)
\geq\epsilon_1\mathcal{P}(\epsilon_1+L_1\epsilon_1),
\end{align*}
where the last inequality is due to \prettyref{eq:infty} and by choosing $\theta$ so that $\dist\left(b_{ij}\left(t,\theta\right),o_k\right)-d_0\leq\epsilon_1$. The lemma is proved by tending $\epsilon_1$ to zero and applying \ref{ass:Barrier}.
\end{proof}

\subsection{\label{sec:PrimalIPM}Feasible Interior-Point Method}
We now combine the above ideas to design a feasible interior point method for the SIP problem. We divide the bounded temporal domain into a disjoint set of intervals, $[0,T]=\bigcup_l[T_0^l,T_1^l]$, and choose the midpoint constraint as the representative. As a result, the penalty functions transform the inequality-constrained NLP into an unconstrained one as follows:
\begin{equation}
\begin{aligned}
\label{eq:penalty}
\argmin{\theta}&\mathcal{E}(\theta)\triangleq\mathcal{O}(\theta)+
\mu\sum_{ijkl}(T_1^l-T_0^l)\mathcal{P}_{ijkl}(\theta)\\
&\mathcal{P}_{ijkl}(\theta)\triangleq\mathcal{P}_{ijk}(\frac{T_0^l+T_1^l}{2},\theta),
\end{aligned}
\end{equation}
where $\mu$ is a positive weight of the barrier coefficient. Note that we weight the penalty function $\mathcal{P}_{ijkl}$ by the time span $T_1^l-T_0^l$ in order to approximate the integral form $\bar{\mathcal{P}}_{ijkl}(\theta)\triangleq\bar{\mathcal{P}}_{ijk}(T_0^l,T_1^l,\theta)$ in the sense of Riemann sum. Standard first- and second-order algorithms can be utilized to solve \prettyref{eq:penalty}, where the search direction of a first-order method $d^{(1)}$ is:
\begin{align*}
d^{(1)}\triangleq-\nabla_\theta\mathcal{E},
\end{align*}
and that of the second-order method is:
\begin{align*}
d^{(2)}\triangleq\adj\left(\nabla_\theta^2\mathcal{E}\right)^{-1}d^{(1)}.
\end{align*}
Here $\adj(\bullet)$ is an modulation function for a Hessian matrix such that $\underline{\beta}I\preceq\adj(H)\preceq\bar{\beta}I$ for some positive constants $\underline{\beta}$ and $\bar{\beta}$. After a search direction is computed, a step size $\alpha$ is adaptively selected to ensure the first Wolfe's condition:
\begin{align}
\label{eq:Wolfe}
\mathcal{E}(\theta+d\alpha)\leq\mathcal{E}(\theta)+c\left<d\alpha,\nabla_\theta\mathcal{E}\right>,
\end{align}
where $c\in(0,1)$ is a positive constant. It has been shown that if the smallest $z\in\mathbb{Z}^+$ is chosen such that $\alpha=1/z$ satisfies \prettyref{eq:Wolfe}, then the feasible interior-point method will converge to the first-order critical point of $\mathcal{E}$ \cite[Proposition~1.2.4]{bertsekas1997nonlinear}. Under the finite-precision arithmetic of a computer, we would terminate the loop of the $\theta$ update when $\|d^{(1)}\|_\infty\leq\epsilon_d$. Furthermore, we add an outer loop to reduce the duality gap by iteratively reducing $\mu$ down to a small constant $\epsilon_\mu$. The overall interior point procedure of solving inequality-constrained NLP is summarized in \prettyref{alg:FIPM}.
\begin{algorithm}[ht]
\caption{\label{alg:FIPM} Feasible Interior Point Method}
\begin{algorithmic}[1]
\Require{Feasible $\theta$, initial $\mu$, $\epsilon_\alpha$, $\epsilon_d$, $\epsilon_\mu$, $\gamma\in(0,1)$}
\Ensure{Locally optimal $\theta$}
\While{$\mu>\epsilon_\mu$}
\State $d\gets d^{(1)}$ or $d\gets d^{(2)}$
\While{$\|d\|_\infty>\epsilon_d$}
\State $\alpha,\epsilon_\alpha\gets$Line-Search($\theta,d,\epsilon_\alpha$)
\State $\theta\gets\theta+d\alpha$\label{ln:Update}
\State $d\gets d^{(1)}$ or $d\gets d^{(2)}$
\EndWhile
\State $\mu\gets\mu\gamma$
\EndWhile
\State Return $\theta$
\end{algorithmic}
\end{algorithm}

\subsection{\label{sec:Subdivision}Adaptive Subdivision}
\prettyref{alg:FIPM} is used to solve NLP instead of SIP. As analyzed in \prettyref{sec:MotionBound}, the feasible domain of NLP derived by surrogate constraints can be larger than that of SIP. To ensure feasibility in terms of semi-infinite constraints, we utilize the motion bound \prettyref{lem:bound} and add an additional safety check in the line search procedure as summarized in \prettyref{alg:safety}. \prettyref{alg:safety} uses a more conservative feasibility condition that shrinks the feasible domain by $\psi(T_1^l-T_0^l)$. The motion bound \prettyref{lem:bound} immediately indicates that $\psi(x)=L_1x/2$. However, our theoretical analysis requires an even more conservative $\psi$ defined as:
\begin{align}
\label{eq:psi}
\psi(x)=L_1x/2+L_2x^\eta,
\end{align}
where $L_2$ and $\eta$ are positive constants. We choose to only accept $\alpha$ found by the line search algorithm when $\theta+d\alpha$ passes the safety check. On the other hand, the failure of a safety check indicates that the surrogate constraint is not a sufficiently accurate approximation of the semi-infinite constraints and a subdivision is needed. We thus adopt a midpoint subdivision, dividing $[T_0^l,T_1^l]$ into two pieces $[T_0^l,(T_0^l+T_1^l)/2]$ and $[(T_0^l+T_1^l)/2,T_1^l]$. This procedure is repeated until $\alpha$ found by the line search algorithm passes the safety check. Note that the failure of safety check can be due to two different reasons: 1) the step size $\alpha$ is too large; 2) more subdivisions are needed. Since the first reason is easier to check and fix, so we choose to always reduce $\alpha$ when safety check fails, until some lower bound of $\alpha$ is reached. We maintain such a lower bound denoted as $\epsilon_\alpha$. Our line-search method is summarized in \prettyref{alg:search}. Our SIP solver is complete by combining \prettyref{alg:FIPM}, \ref{alg:safety}, and \ref{alg:search}.
\begin{algorithm}[ht]
\caption{\label{alg:safety} Safety-Check($\theta$)}
\begin{algorithmic}[1]
\Ensure{$\left<i,j,k,l\right>$ such that $\mathcal{P}_{ijkl}$ violates safety condition}
\For{Each penalty term $\mathcal{P}_{ijkl}$}
\If{$\dist\left(b_{ij}\left(\frac{T_0^l+T_1^l}{2},\theta\right),o_k\right)\leq d_0+\psi(T_1^l-T_0^l)$}
\State Return $\left<i,j,k,l\right>$
\EndIf
\EndFor
\State Return None
\end{algorithmic}
\end{algorithm}

\begin{algorithm}[ht]
\caption{\label{alg:search} Line-Search($\theta,d,\epsilon_\alpha$)}
\begin{algorithmic}[1]
\Require{Initial $\alpha_0$, $\gamma\in(0,1)$}
\Ensure{Step size $\alpha$ and updated $\epsilon_\alpha$}
\State $\alpha\gets\alpha_0$
\State $\theta'\gets\theta+d\alpha$
\State $\left<i,j,k,l\right>\gets$Safe-Check($\theta'$)
\While{$\left<i,j,k,l\right>\neq$None $\lor$ $\theta'$ violates \prettyref{eq:Wolfe}}
\If{$\left<i,j,k,l\right>\neq$None}
\If{$\alpha\leq\epsilon_\alpha$}
\State $\epsilon_\alpha\gets\gamma\epsilon_\alpha$
\State Subdivide($\mathcal{P}_{ijkl}$) and re-evaluate $\mathcal{E}(\theta)$
\State $d\gets d^{(1)}$ or $d\gets d^{(2)}$
\Else
\State $\alpha\gets\gamma\alpha$
\EndIf
\Else
\State $\alpha\gets\gamma\alpha$
\EndIf
\State $\theta'\gets\theta+d\alpha$
\State $\left<i,j,k,l\right>\gets$Safe-Check($\theta'$)
\EndWhile
\State Return $\alpha, \epsilon_\alpha$
\end{algorithmic}
\end{algorithm}
\section{Convergence Analysis}
In this section, we argue that our~\prettyref{alg:FIPM} is suited for solving SIP problems~\prettyref{eq:prob} by establishing three properties. First, the following result is straightforward and shows that our algorithm generates feasible iterations:
\begin{theorem}
\label{thm:feasible}
Under \prettyref{ass:Spatial} and \prettyref{ass:Bounded}, \revised{if \prettyref{alg:FIPM} can find a positive $\alpha$ and update $\theta$ in~\prettyref{ln:Update} during an iteration, then the updated $\theta$ is a feasible solution to \prettyref{eq:prob}.}
\end{theorem}
\begin{proof}
A step size generated by \prettyref{alg:search} must pass the safety check, which in turn ensures that:
\begin{align*}
&\dist\left(b_{ij}\left(\frac{T_0^l+T_1^l}{2},\theta\right),o_k\right)\\
\geq&d_0+\psi(T_1-T_0)>d_0+L_1\frac{T_1^l-T_0^l}{2},
\end{align*}
where we have used our choice of $\psi$ in \prettyref{eq:psi}. From \prettyref{lem:bound}, we have for any $t\in[T_0^l,T_1^l]$ that:
\begin{align*}
&\dist\left(b_{ij}\left(t,\theta\right),o_k\right)\\
\geq&\dist\left(b_{ij}\left(\frac{T_0^l+T_1^l}{2},\theta\right),o_k\right)-L_1\left|t-\frac{T_0^l+T_1^l}{2}\right|>d_0.
\end{align*}
Since \prettyref{alg:search} would check every spatial-temporal constraint subset, the proof is complete.
\end{proof}
\prettyref{thm:feasible} depends on the fact that \prettyref{alg:FIPM} does generate an iteration after a finite amount of computation. However, the finite termination of \prettyref{alg:FIPM} is not obvious for two reasons. First, the line search \prettyref{alg:search} can get stuck in the while loop and never pass the safety check. Second, even if the line search algorithm always terminate finitely, the inner while loop in \prettyref{alg:FIPM} can get stuck forever. This is because a subdivision would remove one and contribute two more penalty terms of form:  $(T_1^l-T_0^l)\mathcal{P}_{ijkl}$ to $\mathcal{E}(\theta)$, which changes the landscape of objective function. As a result, it is possible for a subdivision to increase $\|d\|_\infty$ and \prettyref{alg:FIPM} can never bring $\|d\|_\infty$ down to user-specified $\epsilon_d$. However, the following result shows that neither of these two cases would happen by a proper choice of $\eta$:
\begin{theorem}
\label{thm:termination}
Under \prettyref{ass:Spatial}, \ref{ass:Bounded}, \ref{ass:Barrier}, and suppose $\eta<1/6$, \prettyref{alg:FIPM} terminates after a finite number of subdivisions.
\end{theorem}
\begin{proof}
See \prettyref{sec:termination}.
\end{proof}
\prettyref{thm:termination} shows the well-definedness of \prettyref{alg:FIPM}, which aims at solving the NLP \prettyref{eq:penalty} instead of the original \prettyref{eq:prob}. Our final result bridges the gap by showing that the first-order optimality condition of \prettyref{eq:penalty} approaches that of \prettyref{eq:prob} by a sufficiently small choice of $\mu$ and $\epsilon_\mu$:
\begin{theorem}
\label{thm:optimality}
We take \prettyref{ass:Spatial}, \ref{ass:Bounded}, \ref{ass:Barrier}, \ref{ass:GMFCQ} and suppose $\eta<1/6$. If we run \prettyref{alg:FIPM} for infinite number of iterations using null sequences $\{\mu^k\}$ and $\{\epsilon_d^k\}$, where $k$ is the iteration number, then we get a solution sequence $\{\theta^k\}$ such that every accumulation point $\theta^0$ satisfies the first-order optimality condition of \prettyref{eq:prob}.
\end{theorem}
\begin{proof}
See \prettyref{sec:optimality}.
\end{proof}
\section{\label{sec:implementation}Realization on Articulated Robots}
We introduce two versions of our method. In our first version, we assume both the robot and the environmental geometries are discretized using triangular meshes. Although triangular meshes can represent arbitrary concave shapes, they requires a large number of elements leading to prohibitive overhead even using the acceleration techniques introduced in~\prettyref{sec:acceleration}. Therefore, our second version reduces the number of geometric primitives by approximating each robot link and obstacle with a single convex hull~\cite{dai2018synthesis,amice2022finding} or multiple convex hulls via a convex decomposition~\cite{lien2004approximate}. In other words, the $b_{ij}$ and $o_k$ in our method can be a moving point, edge, triangle, or general convex hull. \revised{In our first version, we need to ensure the two triangle meshes are collision-free. To this end, it suffices to ensure the distances between every pair of edges and every pair of vertex and triangle are larger than $d_0$~\cite{1730806}.} In our second version, we need to ensure the distance between every convex-convex pair is larger than $d_0$. However, it is known that edge-edge or convex-convex distance functions are not differentiable. We follow~\cite{6710113} to resolve this problem by bulging each edge or convex hull using curved surfaces, making them strictly convex with well-defined derivatives. In this section, we present technical details for a practical realization of our method to generate trajectories of articulated robots with translational and hinge rotational joints.

\begin{figure*}[th]
\centering
\scalebox{1.0}{
\setlength{\tabcolsep}{2px}
\begin{tabular}{cccc}
\frame{\includegraphics[height=.24\linewidth,trim=20cm 5cm 23cm 10cm,clip]{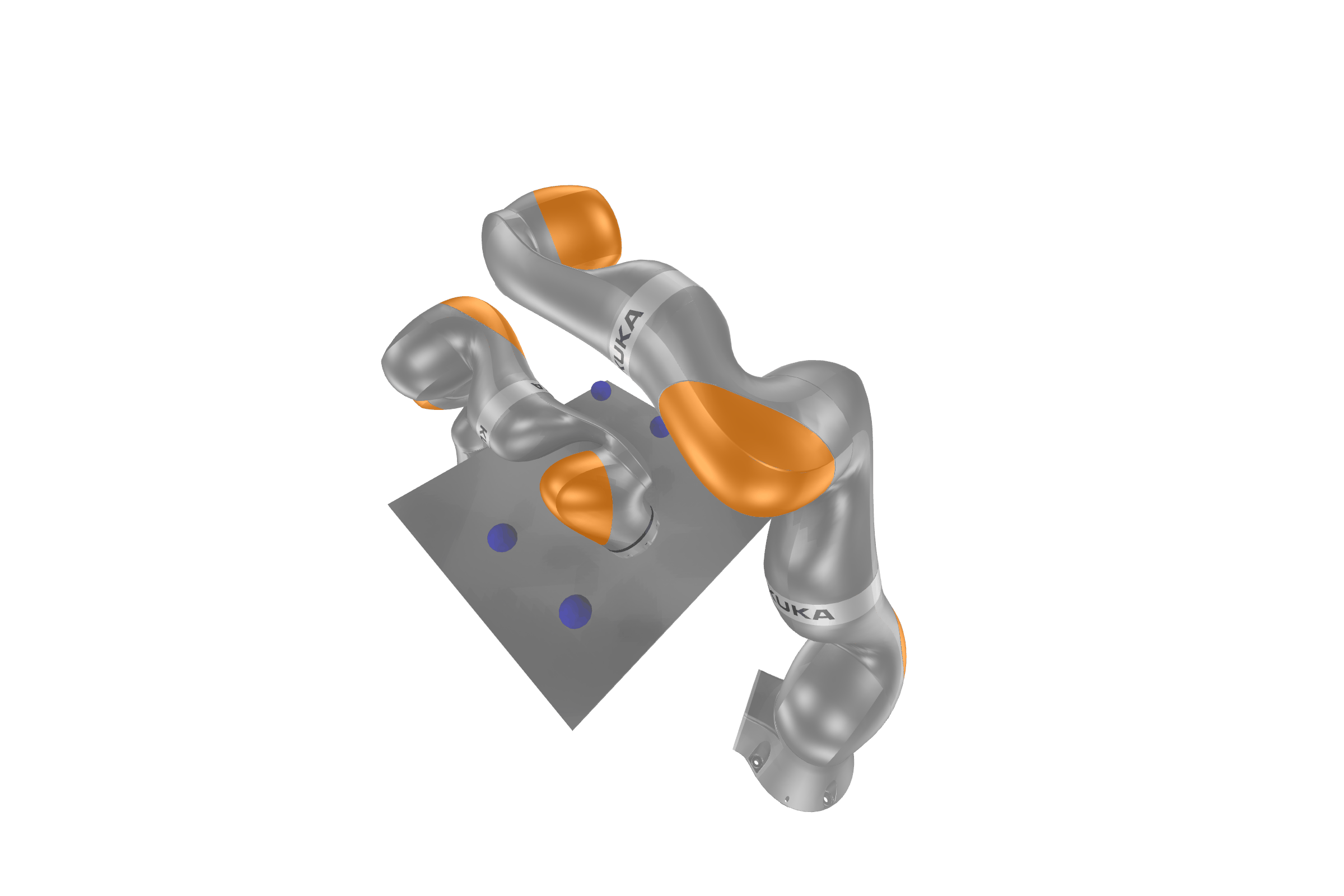}}
\put(-75,-10){(a) dimension=12} &
\frame{\includegraphics[height=.24\linewidth,trim=5cm 8cm 5cm 3cm,clip]{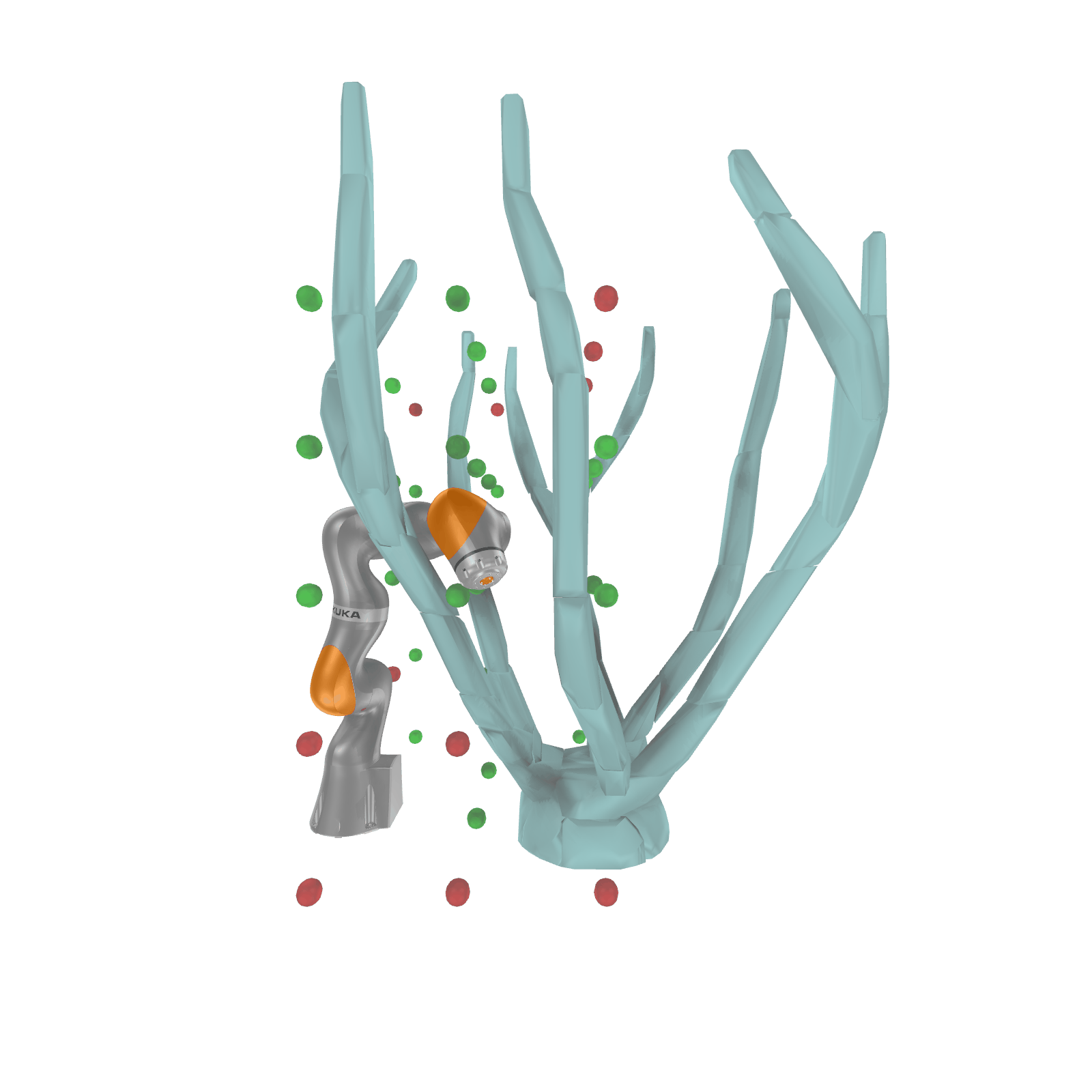}}
\put(-75,-10){(b) dimension=6} &
\frame{\includegraphics[height=.24\linewidth,trim=11cm 3cm 10cm 0cm,clip]{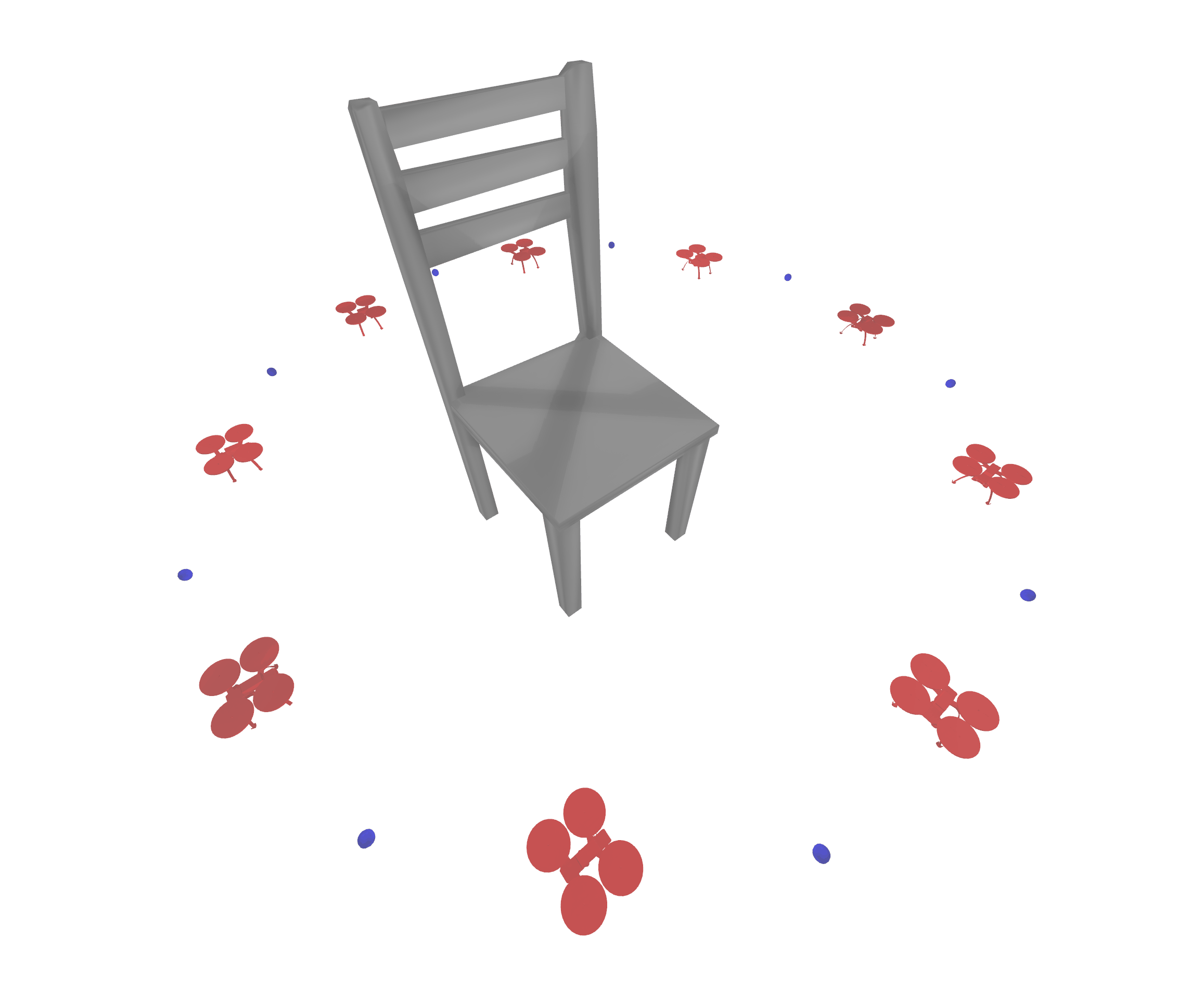}}
\put(-75,-10){(c) dimension=54} &
\frame{\includegraphics[height=.24\linewidth,trim=0cm 4cm 8cm 4cm,clip]{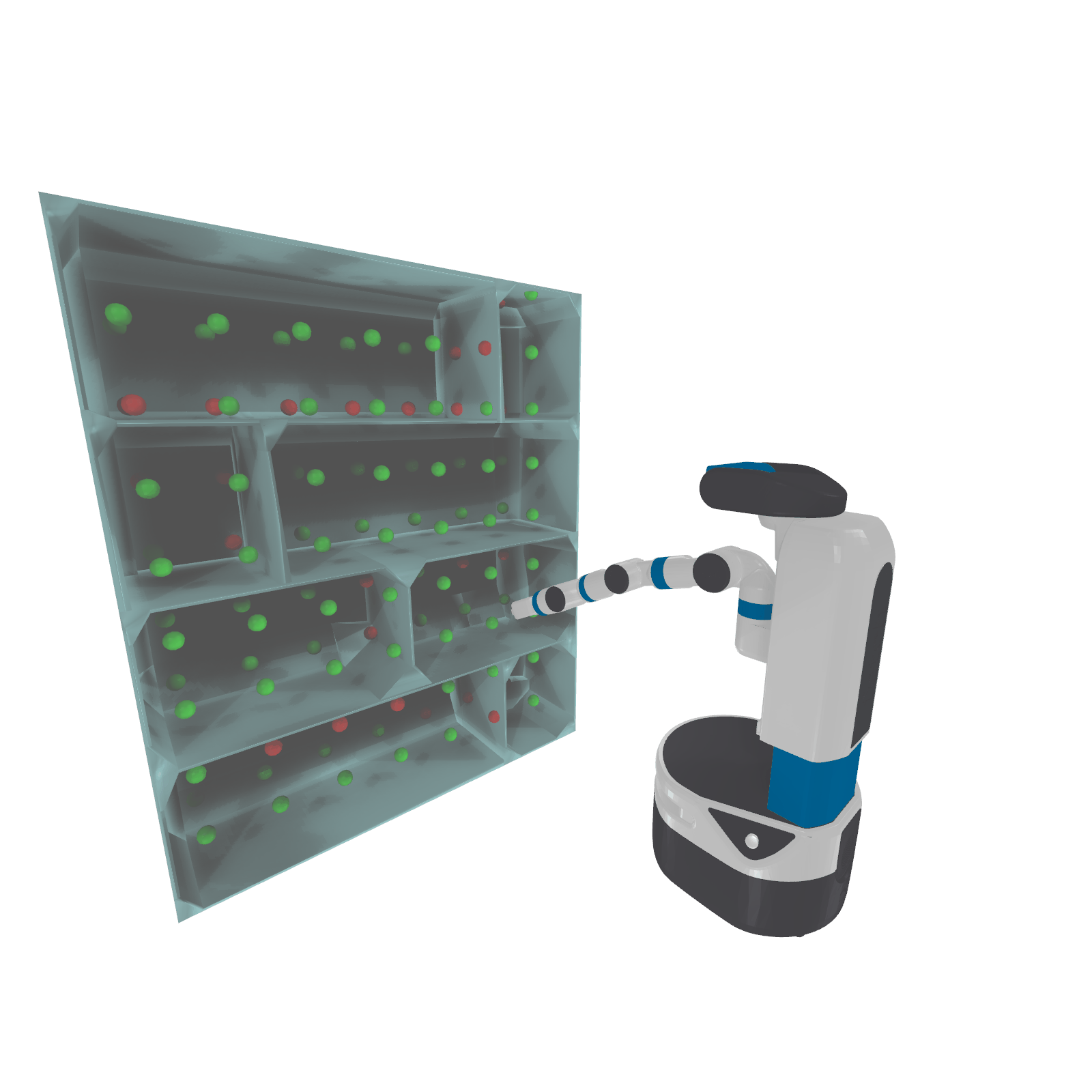}}
\put(-75,-10){(d) dimension=11}
\end{tabular}}
\caption{\label{fig:examples}Snapshots of our four benchmark problems with labeled dimension of configuration spaces.}
\end{figure*}
\subsection{\label{sec:L1}Computing Lipschitz Upper Bound}
Our method requires the Lipschitz constant $L_1$ to be evaluated for each type of geometric shape. \revised{We denote by $L_1^{ijk}$ as the Lipschitz constant for the pair of $b_{ij}$ and $o_k$ satisfying~\prettyref{lem:bound}. We first consider the case with $b_{ij}$ being a moving point, and all other cases are covered by minor modifications. We denote $\Theta$ as the vector of joint parameters, which is also a function of $t$ and $\theta$. By the chain rule, we have:
\begin{align*}
L_1^{ijk}=&\max\left|\FPP{\dist(b_{ij}(t,\theta),o_k)}{t}\right|\\
=&\max\left|\FPP{\dist(b_{ij}(t,\theta),o_k)}{b_{ij}(t,\theta)}\FPP{b_{ij}(t,\theta)}{\Theta(t,\theta)}\FPP{\Theta(t,\theta)}{t}\right|.
\end{align*}
Without a loss of generality, we assume each entry of $\FPPR{\Theta(t,\theta)}{t}$ is limited to the range $[-1,1]$. The first term is a distance function, whose subgradient has a norm at most 1~\cite{osher2005level}. These results combined, we derive the following upper bound for $L_1^{ijk}$ independent of $o^k$, which in turn is denoted as $L_1^{ij}$:}
\begin{align*}
L_1^{ij}\leq\max\left|\FPP{b_{ij}(t,\theta)}{\Theta(t,\theta)}\right|_{2,1},
\end{align*}
which means $L_1^{ij}$ is upper bound of the $l_{2,1}$-norm of the Jacobian matrix (the sum of $l_2$-norm of each column). Next, we consider the general case with $b_{ij}$ being a convex hull with $N$ vertices denoted as $b_{ij}^{1,\cdots,N}$. We have the closest point on $b_{ij}$ to $o_k$ lies on some interpolated point $\sum_{m=1}^Nb_{ij}^m\xi^m$, where $\xi^m$ are convex-interpolation weights that are also function of $\theta(t)$. \revised{We have the following upper-bound for the distance variation over time:
\begin{align*}
&\left|\text{dist}(\sum_{m=1}^{N}b_{ij}^m(t_1,\theta)\xi_1^m,o_k)-\text{dist}(\sum_{m=1}^{N}b_{ij}^m(t_2,\theta)\xi_2^m,o_k)\right|\\
\leq&\left|\text{dist}(\sum_{m=1}^{N}b_{ij}^m(t_1,\theta)\xi_1^m,o_k)-\text{dist}(\sum_{m=1}^{N}b_{ij}^m(t_2,\theta)\xi_1^m,o_k)\right|\\
\leq&L_1^{ij}(\xi_1^m)|t_1-t_2|.
\end{align*}
The inequality above is due to the fact that coefficients $\xi^m$ minimize the distance, so replacing $\xi_2^m$ with $\xi_1^m$ will only increase the distance. Here we abbreviate $\xi_\bullet^m\triangleq\xi^m(\theta(t_\bullet))$ and assume that $\text{dist}(b_{ij}(t_1,\theta),o_k)<\text{dist}(b_{ij}(t_2,\theta),o_k)$, and we can switch $t_1$ and $t_2$ otherwise. Next, we treat $\xi_1^m$ as a constant independent of $\theta(t)$ and estimate the $\xi_1^m$-dependent Lipschitz constant $L_1^{ij}(\xi_1^m)$ as:
\small
\begin{align*}
L_1^{ij}(\xi_1^m)=&\max\left|\FPP{\dist(b_{ij}(t,\theta),o_k)}{\sum_{m=1}^Nb_{ij}^m\xi_1^m}
\FPP{\sum_{m=1}^Nb_{ij}^m\xi_1^m}{\Theta(t,\theta)}\FPP{\Theta(t,\theta)}{t}\right|\\
\leq&\max\sum_{m=1}^N\left|\FPP{b_{ij}^m(t,\theta)}{\Theta(t,\theta)}\right|_{2,1}\xi_1^m
\leq\max_{m=1,\cdots,N}\left|\FPP{b_{ij}^m(t,\theta)}{\Theta(t,\theta)}\right|_{2,1}.
\end{align*}
\normalsize
where the last inequality is due to the fact that $\xi_1^m$ form a convex combination. We see that our estimate of $L_1^{ij}(\xi_1^m)$ is indeed independent of $\xi_1^m$ and can be re-defined as our desired Lipschitz constant $L_1^{ij}$. In other words, the Lipschitz constant of a moving convex hull is the maximal Lipschitz constant over its vertices, and the Lipschitz constants of edge and triangle are just special cases of a convex hull.}

\begin{wrapfigure}{r}{0.22\textwidth}
\centering
\includegraphics[width=0.2\textwidth]{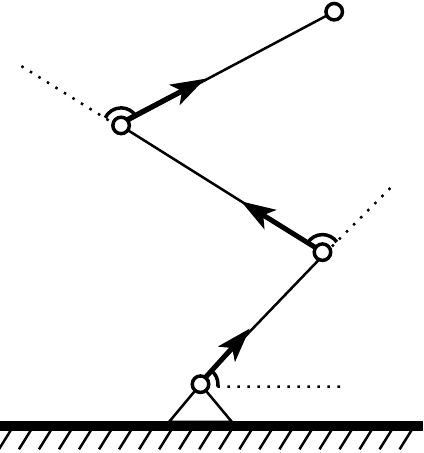}
\put(-40,20){\small$\Theta_1$}
\put(-30,60){\small$\Theta_2$}
\put(-80,90){\small$\Theta_3$}
\put(-50,35){\small$l_1$}
\put(-50,70){\small$l_2$}
\put(-20,95){\small$l_3$}
\end{wrapfigure}
It remains to evaluate the upper bound of the $l_{2,1}$-norm of the Jacobian matrix for a moving point $b_{ij}$. We can derive this bound from the forward kinematic function. For simplicity, we assume a robot arm with only hinge joints as illustrated in the inset. We use $\Theta_k$ to denote the angle of the $k$th hinge joint. If $b_{ij}$ lies on the $K$th link, then only $\Theta_{1,\cdots,K}$ can affect the position of $b_{ij}$. We assume the $k$th link has length $l_k$, then the maximal influence of $\Theta_k$ on $b_{ij}$ happens when all the $k,\cdots,K$th links are straight, so that:
\begin{align*}
\left|\FPP{b_{ij}}{\Theta_k}\right|\leq\sum_{m=1}^kl_m\Longrightarrow 
\left|\FPP{b_{ij}}{\Theta}\right|_{2,1}\leq\sum_{k=1}^K\sum_{m=k}^Kl_m.
\end{align*}

\subsection{High-Order Polynomial Trajectory Parameterization}
Our $L_1^{ij}$ formulation relies on the boundedness of $\FPPR{\Theta(t,\theta)}{t}$. And articulated robots can have joint limits which must be satisfied at any $t\in[0,T]$. To these ends, we use high-order composite B\'ezier curves to parameterize the trajectory $\Theta(t,\theta)$ in the configuration space, so that $\Theta(t,\theta)$ is a high-order polynomial function. In this form, bounds on $\Theta(t,\theta)$ at an arbitrary $t$ can be transformed into bounds on its control points~\cite{shikin1995handbook}. We denote the lower- and upper-joint limits as $\underline{\Theta}$ and $\bar{\Theta}$, respectively. If we denote by $M_k$ the matrix extracting the control points of $\Theta_k$ and $M_{ik}$ the $i$th row of $M_k$, then the joint limit constraints can be conservatively enforced by the following barrier function:
\begin{align}
\label{eq:jointLimit}
\sum_i\sum_k\mathcal{P}(\bar{\Theta}_k-M_{ik}\Theta_k)+\mathcal{P}(-\underline{\Theta}_k+M_{ik}\Theta_k).
\end{align}
A similar approach can be used to bound $\FPPR{\Theta_k(t,\theta)}{t}$ to the range $[-1,1]$. We know that the gradient of a B\'ezier curves is another B\'ezier curves with a lower-order, so we can denote by $M_k'$ the matrix extracting the control points of $\FPPR{\Theta_k(t,\theta)}{t}$ and $M_{ik}'$ the $i$th row of $M_k'$. The boundedness of $\FPPR{\Theta(t,\theta)}{t}$ for any $t$ can then be realized by adding the following barrier function:
\begin{align}
\label{eq:gradientBound}
\sum_i\sum_k\mathcal{P}(1-M_{ik}'\Theta_k)+\mathcal{P}(1+M_{ik}'\Theta_k).
\end{align}
Note that these constraints are strictly conservative. However, one can always use more control points in the B\'ezier curve composition to allow an arbitrarily long trajectory of complex motions.

\subsection{\label{sec:acceleration}Accelerated \& \revised{Adaptive} Computation of Barrier Functions}
A naive method for computing the barrier function terms $\sum_{ijkl}\mathcal{P}_{ijkl}$ could be prohibitively costly, and we propose several acceleration techniques that is compatible with our theoretical analysis. Note first that our theoretical results assume the same $L_1$ for all $b_{ij}$, but the $L_1^{ij}$ constant computed in~\prettyref{sec:L1} is different for each $b_{ij}$. Instead of letting $L_1=\fmax{ij}L_1^{ij}$, we could use a different $\phi(x)=L_1^{ij}x/2+L_2x^\eta$ for each $b_{ij}$, leading to a loose safety condition and less subdivisions. Further, note that our potential function $\mathcal{P}$ is locally supported by design and we only need to compute $\mathcal{P}_{ijk}$ if the distance between $b_{ij}$ and $o_k$ is less than \revised{$x_0+d_0$}. We propose to build a spatial-temporal, binary-tree-based bounding volume hierarchy (BVH)~\cite{gu2013efficient} for pruning unnecessary $\mathcal{P}_{ijk}$ terms, where each leaf node of our BVH indicates a unique tuple $<b_{ij},T_0^l,T_1^l>$, which can be checked against each obstacle $o_k$ to quickly prune $<b_{ij},o_k>$ pairs with distance larger than $x_0+d_0$. \revised{The BVH further provides a convenient data-structure to perform adaptive subdivision. When safety check fails for the term $\mathcal{P}_{ijk}$, we only subdivide that single term without modifying the subdivision status of other $b_{ij}$ and $o_k$ pairs. Specifically, a leaf node tuple $<b_{ij},T_0^l,T_1^l>$ is replaced by an internal node with two children: $<b_{ij},T_0^l,(T_0^l+T_1^l)/2>$ and $<b_{ij},(T_0^l+T_1^l)/2,T_1^l>$ upon subdivision.}

\subsection{Handling Self-Collisions}
\revised{Our method inherently applies to handle self-collisions. Indeed, for two articulated robot subsets $b_{ij}$ and $b_{i'j'}$, we have the following generalized motion bound:
\begin{align*}
&|\text{dist}(b_{ij}(t_1,\theta),b_{i'j'}(t_1,\theta))-\text{dist}(b_{ij}(t_2,\theta),b_{i'j'}(t_2,\theta))|\\
\leq&|\text{dist}(b_{ij}(t_1,\theta),b_{i'j'}(t_1,\theta))-\text{dist}(b_{ij}(t_2,\theta),b_{i'j'}(t_1,\theta))|+\\
&|\text{dist}(b_{ij}(t_2,\theta),b_{i'j'}(t_1,\theta))-\text{dist}(b_{ij}(t_2,\theta),b_{i'j'}(t_2,\theta))|\\
\leq&(L_1^{ij}+L_1^{i'j'})|t_1-t_2|,
\end{align*}
where the second inequality is derived by treating $b_{i'j'}(t_1,\theta)$ and $b_{ij}(t_2,\theta)$ as a static obstacle in the first and second term, respectively. The above result implies that if~\prettyref{lem:bound} holds for distances to static obstacles, it also holds for distances between two moving robot subsets by summing up the Lipschitz constants. As a result, we can use the following alternative definition of $\phi(x)$ within the safety check to prevent self-collisions:
\begin{align*}
\phi(x)=(L_1^{ij}+L_1^{i'j'})x/2+L_2x^\eta.
\end{align*}
Readers can verify that all our theoretical results follow for self-collisions by the same argument, and we omit their repetitive derivations for brevity. Notably, using our adaptive subdivision scheme introduced in~\prettyref{sec:acceleration}, the subdivision status of $b_{ij}$ and $b_{i'j'}$ can be different. For example, $b_{ij}$ can have a subdivision interval $[T_0^l,T_1^l]$, while $b_{i'j'}$ has an overlapping interval $[T_0^{l'},T_1^{l'}]$ such that $[T_0^l,T_1^l)\cap[T_0^{l'},T_1^{l'})\neq\emptyset$, but $[T_0^l,T_1^l]\neq[T_0^{l'},T_1^{l'}]$. Since we use the midpoint constraint as the representative of the interval, there is no well-defined midpoint for such inconsistent interval pairs. To tackle this issue, we note that by the midpoint subdivision rule, we have either $[T_0^l,T_1^l]\subset[T_0^{l'},T_1^{l'}]$ or $[T_0^{l'},T_1^{l'}]\subset[T_0^l,T_1^l]$. As a result, we could recursively subdivide the larger interval until the two intervals are identical.}
\section{Evaluation}
We implement our method using C++ and evaluate the performance on a single desktop machine with one 32-core AMD 3970X CPU. We make full use of the CPU cores to parallelize the BVH collision check, energy function, and derivative computations. For all our experiments, we use $x_0=10^{-3}$, $L_2=10^{-4}$, $\eta=1/7$, $\mu=10^{-2}$, $\epsilon_d=10^{-4}$. The trajectory is parameterized in configuration space using a $5$th-order composite B\'ezier curve with $5$ segments over a horizon of $T=5$s. We use four computational benchmarks discussed below.

\subsection{Benchmark Problems}
\begin{figure}[ht]
\centering
\setlength{\tabcolsep}{2px}
\begin{tabular}{cc}
\includegraphics[height=.3\linewidth]{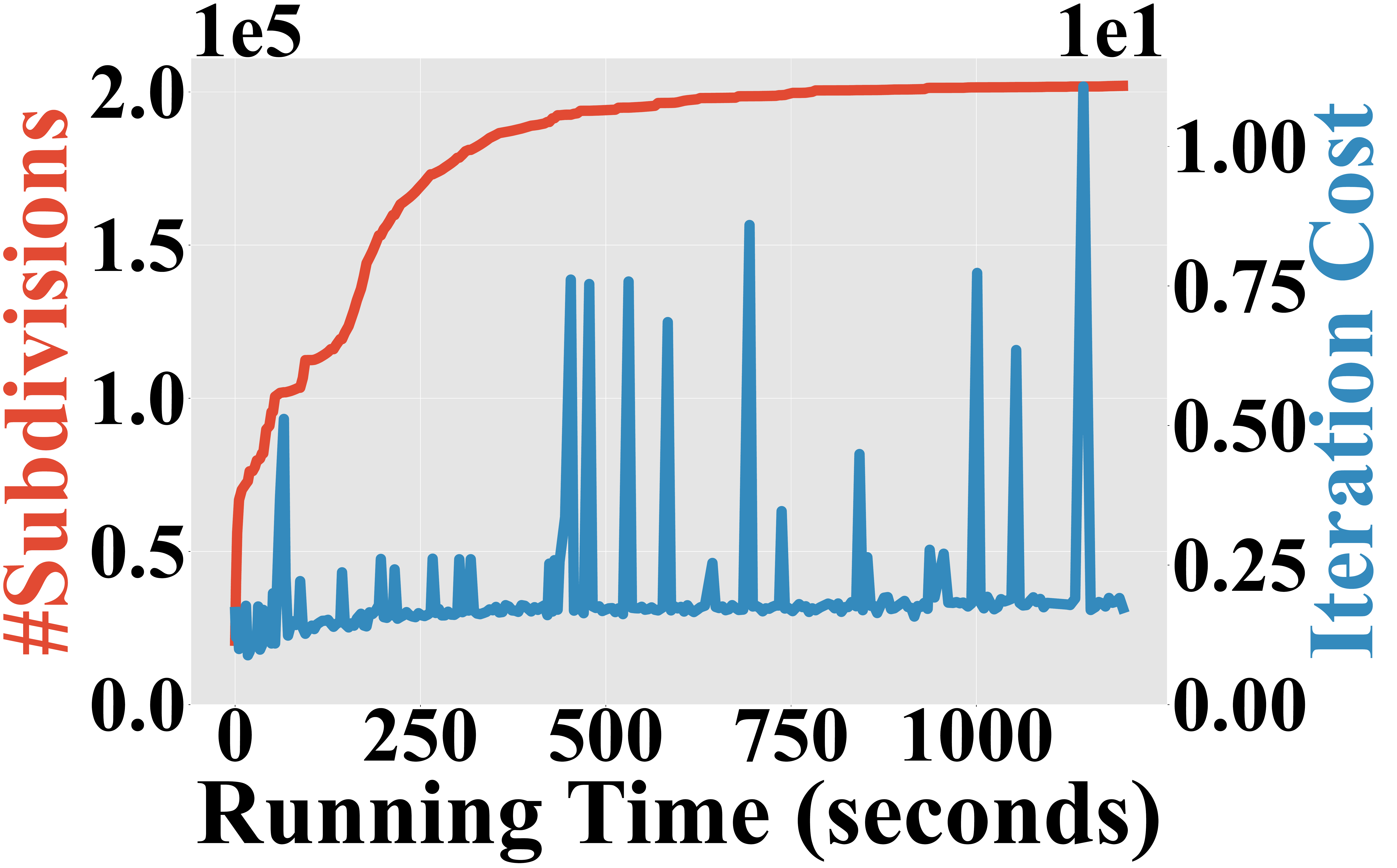}
\put(-10,-3){(a)} &
\includegraphics[height=.3\linewidth]{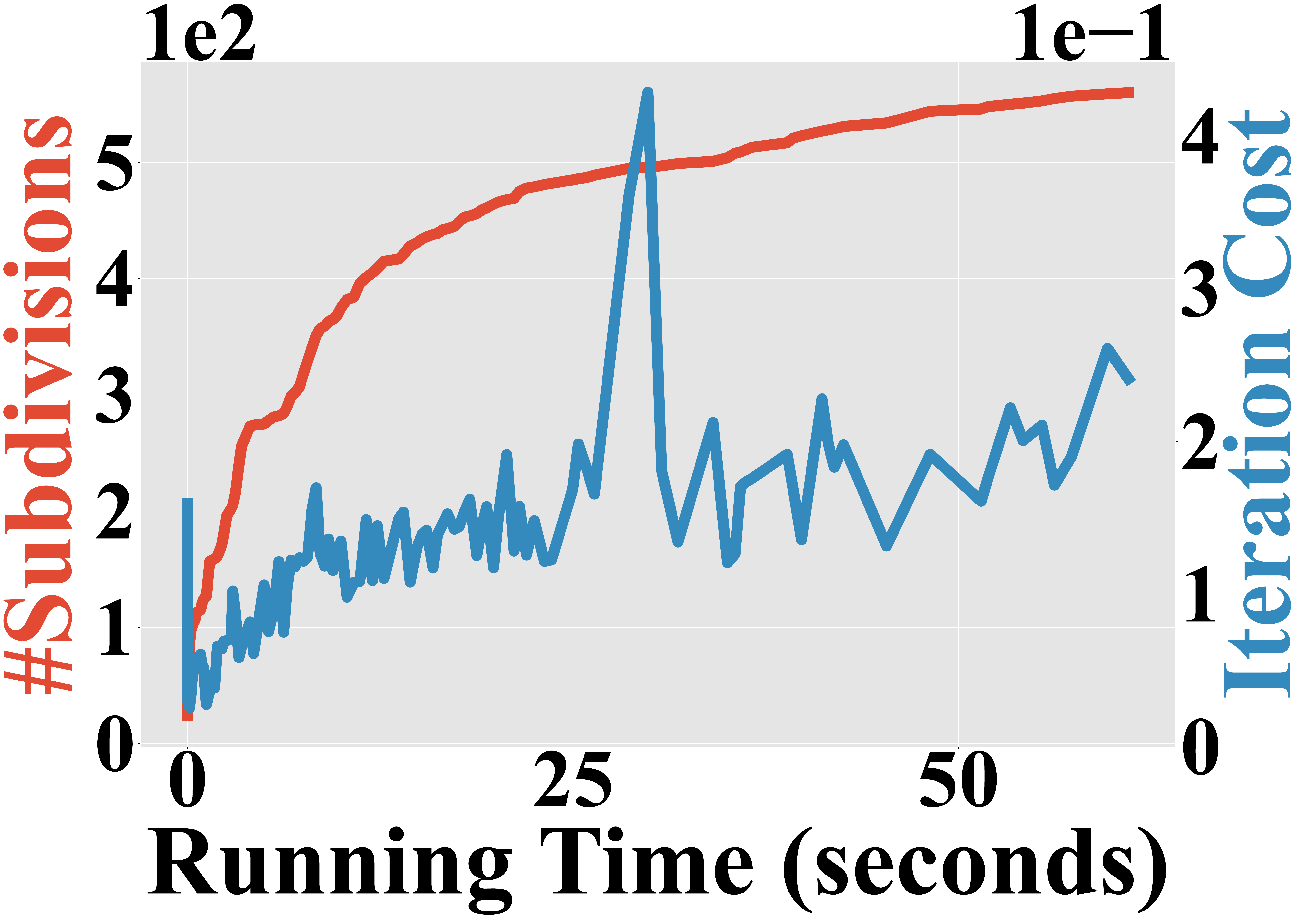}
\put(-10,-3){(b)}
\end{tabular}
\caption{\label{fig:conv1} The number of subdivisions and cost per iteration in seconds plotted against the computational time for our first benchmark using triangular mesh (a) and convex hull (b) representations.}
\end{figure}
Our first benchmark (\prettyref{fig:examples}a) involves two LBR iiwa robot arms simultaneously reaching 4 target points in a shared workspace. To this end, our objective function involves a distance measure between the robot end-effectors and the target points, as well as a Laplacian trajectory smoothness metric as in~\cite{park2012itomp}. The convergence history as well as the number of subdivisions is plotted in~\prettyref{fig:conv1} for both versions of geometric representations: triangular mesh and convex hull. Our method using triangular meshes is much slower than that using convex hulls, due to the fine geometric details leading to a large number of triangle-triangle pairs. Our method with the convex hull representation converges after 359 subdivisions, 231 iterations, and 2.23 minutes of computation to reach all 4 target positions. We have also plotted the cost per iteration in~\prettyref{fig:conv1}, which increases as more subdivisions and barrier penalty terms are introduced.

Our second benchmark (\prettyref{fig:examples}b with successful points in green and unsuccessful points in red) involves a single LBR iiwa robot arm interacting with a tree-like obstacle with thin geometric objects. Such obstacles can lead to ill-defined gradients for infeasible discretization methods or even tunnel through robot links~\cite{doi:10.1177/0278364914528132}, while our method can readily handle such ill-shaped obstacles. We sample a grid of target positions for the end-effector to reach and run our algorithm for each position. Our method can successfully reach 16/56 positions. We further conduct an exhaustive search for each unsuccessful point. Specifically, if an unsuccessful point is neighboring a successful point, we optimize an additional trajectory for the end-effector to connect the two points. Such connection is performed until no new points can be reached. In this way, our method can successfully reach 39/56 positions. On average, to reach each target point, our method with the convex hull representation converges after 416 subdivisions, 246 iterations, and 0.31 minutes of computation, so the total computational time is 17.36 minutes.

\begin{figure}[ht]
\centering
\setlength{\tabcolsep}{2px}
\begin{tabular}{cc}
\includegraphics[width=.45\linewidth]{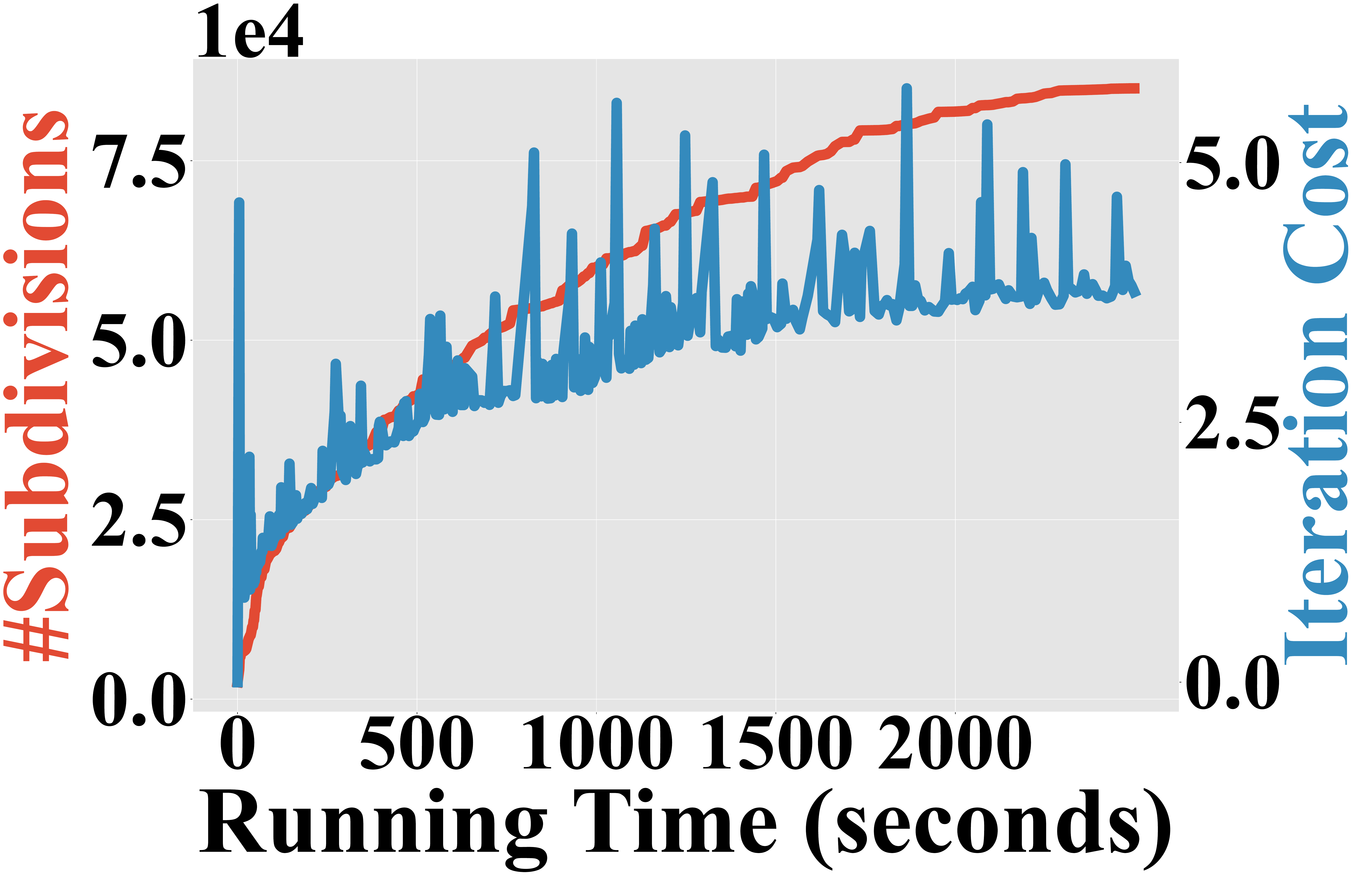}
\put(-10,-3){(a)} &
\includegraphics[width=.45\linewidth]{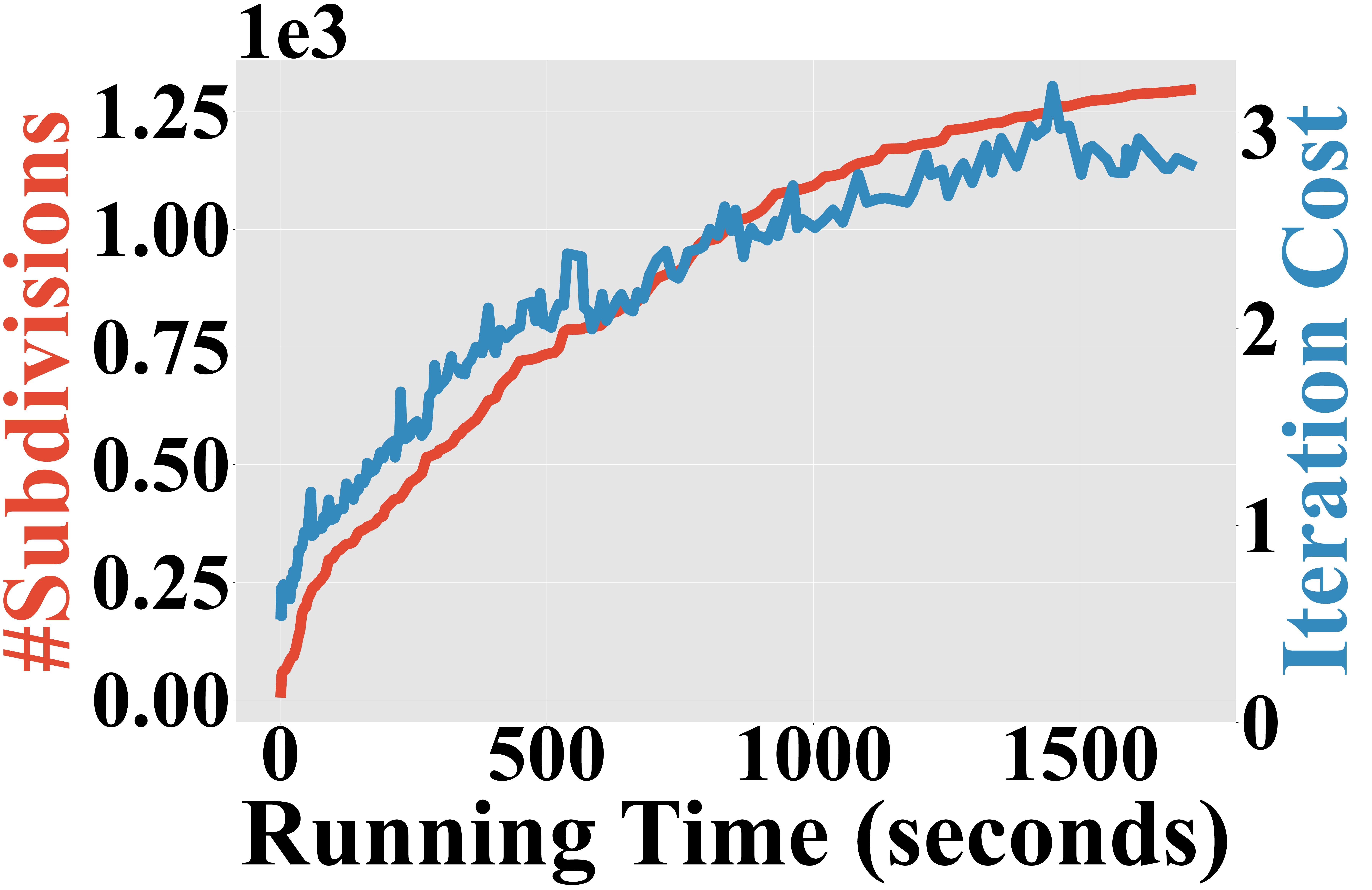}
\put(-10,-3){(b)}
\end{tabular}
\caption{\label{fig:conv3} The number of subdivisions and cost per iteration in seconds plotted against the computational time for our third benchmark using triangular mesh (a) and convex hull (b) representations.} 
\end{figure}
Our third benchmark (\prettyref{fig:examples}c) involves a swarm of UAVs navigating across each other through an obstacle. A UAV can be modeled as a free-flying rigid body with $6$ degrees of freedom. Prior work~\cite{oleynikova2016continuous} searches for only the translations and then exploits differential flatness~\cite{5980409} to recover feasible orientation trajectories. However, such orientation trajectories might not be collision-free. Instead, we can optimize both the translation and rotation with a collision-free guarantee. We have experimented with both geometric representations, and the convergence history of this example is plotted~\prettyref{fig:conv3}. Again, our method with triangular mesh can represent detailed geometries but takes significantly more computation time. In comparison, our method with the convex hull representation converges after 1297 subdivisions, 188 iterations, and 25.87 minutes of computation.
\begin{table}[ht]
\centering
\begin{tabular}{ll}
\toprule
Step & Time  \\
\midrule
Line-search direction computation & 1.6\%\\
Subdivision & 3.3\%\\
Safety-check & 14.4\%\\
Objective function evaluation & 80.7\%\\
\bottomrule
\end{tabular}
\caption{\label{table:timeProportion} Over an optimization, we summarize the fraction of computation on each component. Our major bottleneck lies in the energy evaluation, taking $80.7\%$ of the computation.}
\end{table}
In~\prettyref{table:timeProportion}, we summarize the fraction of computation for each component of our method for both of the triangle mesh and convex hull representations. Our major bottleneck lies in the energy evaluation, i.e. computing $\mathcal{E}(\theta)$ and its derivatives. Note that we use either triangular mesh or convex hull as our geometric representation. In the former case, each distance term is between a pair of edges or a pair of point and triangle. Although such distance function is cheap to compute, the number of $\mathcal{P}_{ijkl}$ terms is large, leading to a major computational burden. In the latter case, the number of $\mathcal{P}_{ijkl}$ terms is much smaller, but each evaluation of $\mathcal{P}_{ijkl}$ involves the non-trivial computation of the shortest distance between two general convex hulls, which is also the computational bottleneck.

Our final benchmark (\prettyref{fig:examples}d) plans for an armed mobile robot to reach a grid of locations on a book-shelf. We use convex decomposition to represent both the robot links and the book-shelf as convex hulls. Starting from a faraway initial guess, our method can guide the robot to reach 88/114 target positions, achieving a success rate of 77.19\%. On average over each target point, our method with the convex hull representation converges after 874 subdivision, 402 iterations, and 1.77 minutes of computation, so the total computational time is 201.78 minutes.

\begin{table}[ht]
\centering
\setlength{\tabcolsep}{3pt}
\begin{tabular}{lcccccccccc}
\toprule
$i$th Link   & 1    & 2    & 3    & 4    & 5    & 6    & 7    & 9    & 10    \\
\midrule
$\delta L_1^i$ & 3.71 & 5.24 & 4.87 & 5.06 & 5.84 & 6.14 & 6.88 & 7.21 & 7.79  \\
\bottomrule
\end{tabular}
\caption{\label{table:overestimation} Overestimation of Lipschitz constant for each link of the armed mobile robot in \prettyref{fig:examples}d (The $1$st link is closest to the root joint).}
\end{table}
\subsection{Conservativity of Lipschitz Constant}
According to~\prettyref{table:timeProportion}, our main computational bottleneck lies in the large number of penalty terms resulting from repeated subdivision. This is partly due to over-estimation of the Lipschitz constant, resulting in an overly conservative motion bound and more subdivisions to tighten the bound. To quantify the over-estimation for the $i$th rigid body, we compare our Lipschitz constant with the following groundtruth tightest bound:
\begin{align*}
L_1^{ij*}\triangleq\argmin{t\in[T_0^l,T_1^l]}\left|\FPP{b_{ij}(t,\theta)}{t}\right|,
\end{align*}
and define the over-estimation metric as $\delta L_1^i\triangleq\max_j L_1^{ij}/L_1^{ij*}$,
where we calculate $L_1^{ij*}$ by sampling time instances within $t\in [T_0^l,T_1^l]$ at an interval of $\delta t=10^{-3}$ and pick the largest value. We initialize 1000 random trajectories lasting for $1$s ($T_1^l-T_0^l=1$) for the armed mobile robot in \prettyref{fig:examples}d and we summarize the average $\Delta L_1^i$ over 1000 cases for each robot link in~\prettyref{table:overestimation}, with the $1$st link being closest to the root joint. The level of over-estimation increases for links further down the kinematic chain, due to the over-estimation of each joint on the chain. As a result, more subdivisions are needed with longer kinematic chains.

\subsection{Comparisons with Exchange Method}
We combine the merits of prior works into a reliable implementation of the exchange method~\cite{lopez2007semi}. During each iteration, the spatio-temporal deepest penetration point between each pair of convex objects is detected and inserted into the finite index set. To detect the deepest penetration point, we densely sample the temporal domain with a finite interval of $\epsilon=10^{-3}$ and compute penetration depth for every pair of robot links and each time instance. We use the exact point-to-mesh distance computation implemented in CGAL~\cite{fabri2009cgal} to compute penetration depth at each time instance, and we adopt the bounding volume hierarchy and branch-and-bound technique to efficiently prune non-deepest penetrated points as proposed in~\cite{doi:10.1177/0278364911406761}. After the index set is updated, an NLP is formulated and solved using the IPOPT software package~\cite{biegler2009large}. We terminate NLP when the inf-norm of the gradient is less than $10^{-4}$, same as for our method. Note that we sample the temporal domain with a finite interval of $\epsilon=10^{-3}$ and we only insert new constraints into the index set and never remove constraints from it. Therefore, our exchange solver is essentially reducing SIP to NLP solver with progressive constraint instantiation, so that our solver pertains to the same feasibility guarantee as an NLP solver. Such is the strongest feasibility guarantee an exchange-based SIP solver can provide, to the best of our knowledge. For fairness of comparison, we only use the collision constraints for both methods, i.e. the joint limit~\prettyref{eq:jointLimit} and gradient bound~\prettyref{eq:gradientBound} are not used in our method for this section.

\begin{figure}[ht]
\centering
\begin{tabular}{cc}
\frame{\includegraphics[width=.45\linewidth,trim=25cm 5cm 25cm 5cm,clip]
{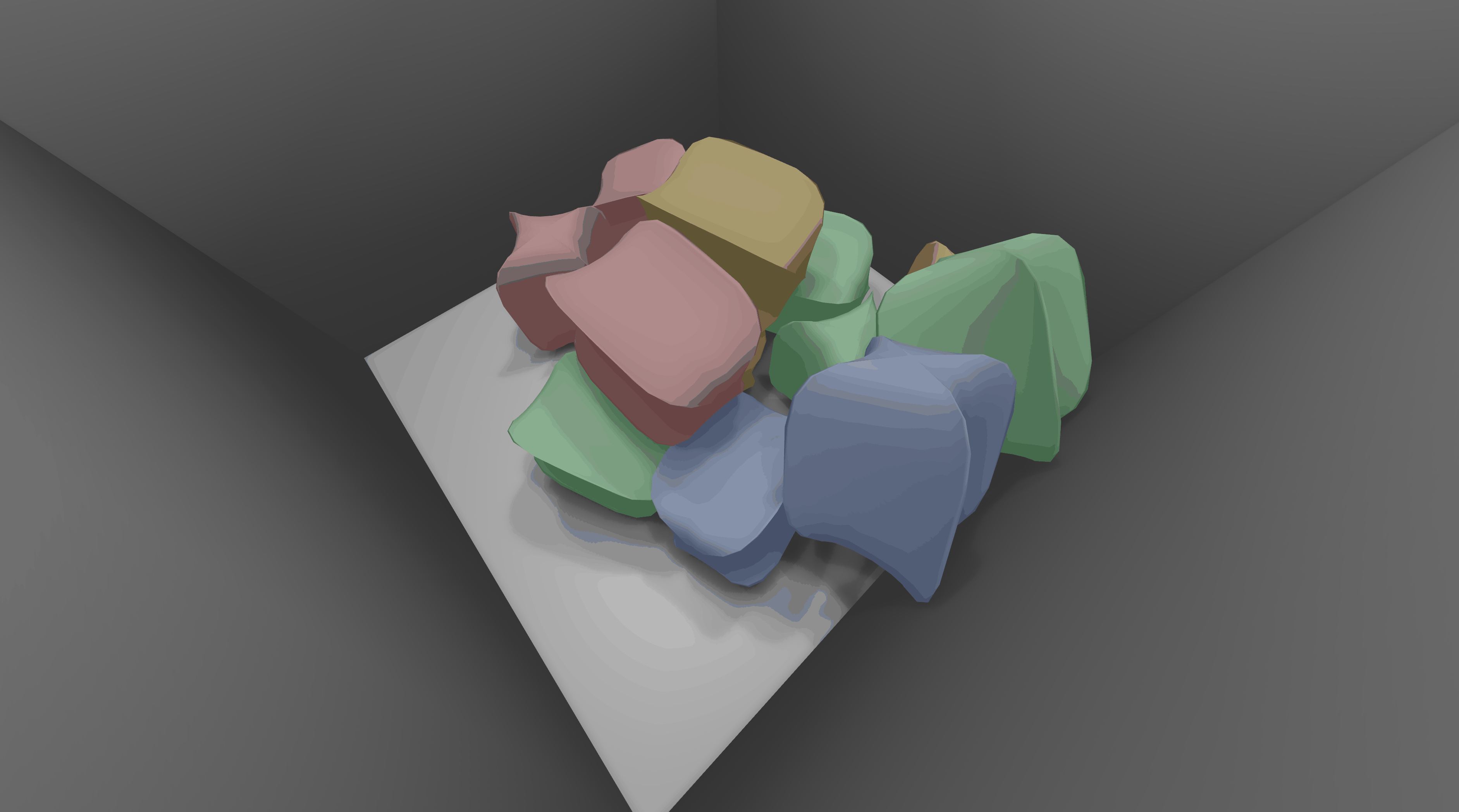}}
\put(-20,5){\textcolor{white}{(a)}}&
\frame{\includegraphics[width=.45\linewidth,trim=25cm 5cm 25cm 5cm,clip]
{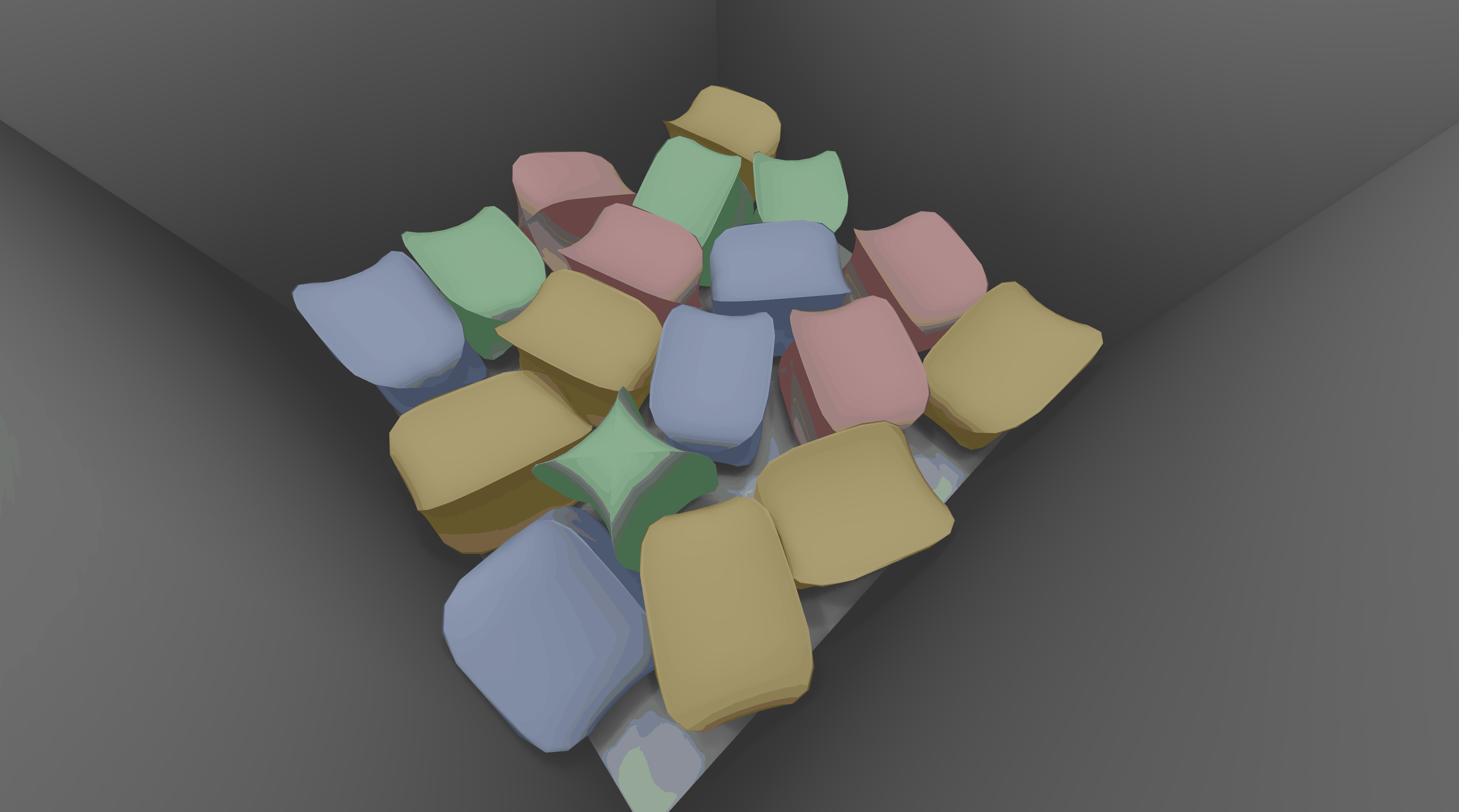}}
\put(-20,5){\textcolor{white}{(b)}}\\
\frame{\includegraphics[width=.45\linewidth,trim=25cm 10cm 25cm 15cm,clip]
{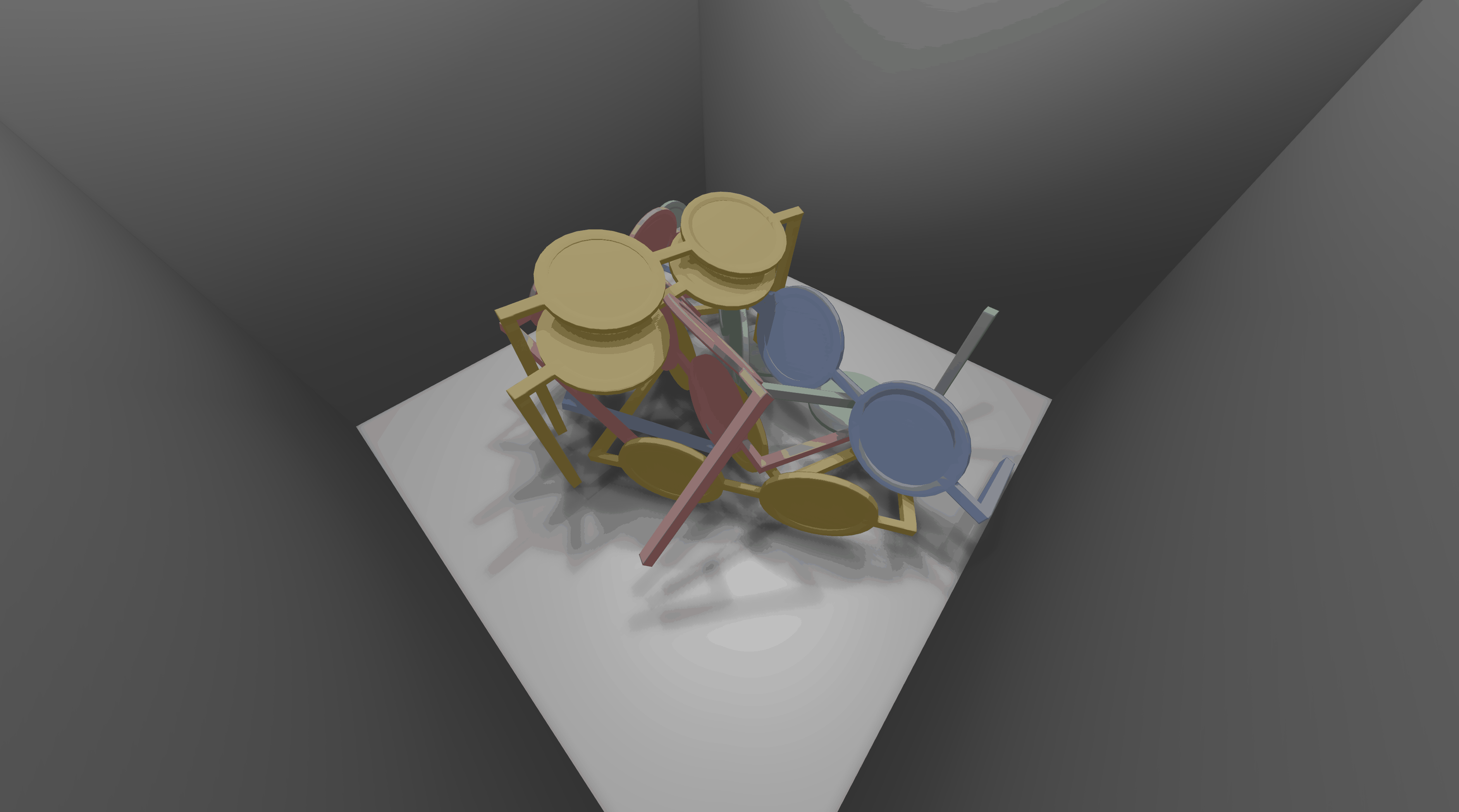}}
\put(-20,5){\textcolor{white}{(c)}}&
\frame{\includegraphics[width=.45\linewidth,trim=25cm 10cm 25cm 15cm,clip]
{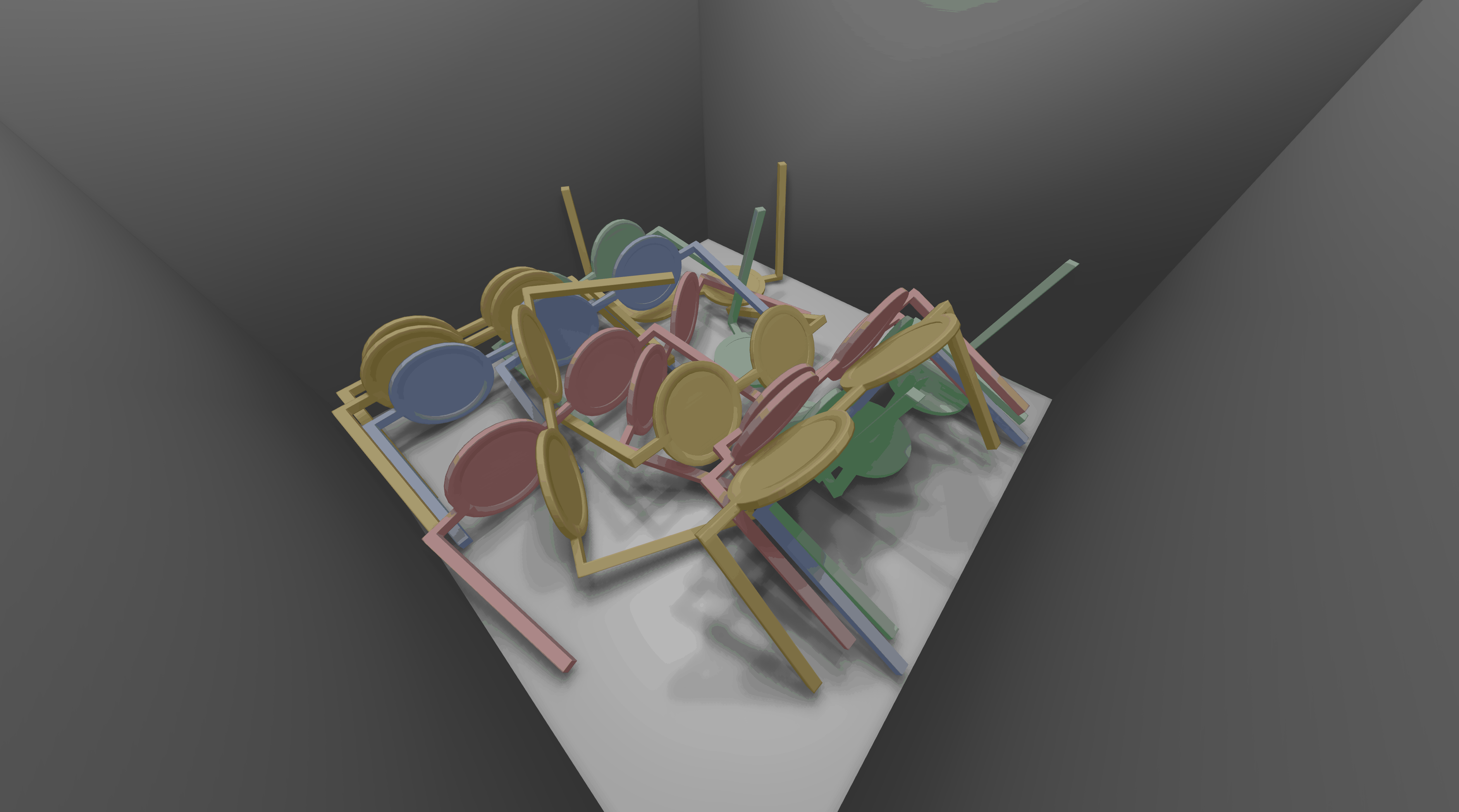}}
\put(-20,5){\textcolor{white}{(d)}}
\end{tabular}
\caption{\label{fig:settle} The result of object settling for 18 bulky objects one-by-one (a) and jointly (b) using our method. Further, we can settle very thin objects such as eyeglasses, either one-by-one (c) or jointly (d).}
\end{figure}
\begin{table}[ht]
\centering
\setlength{\tabcolsep}{2pt}
\begin{tabular}{lcccccccccc}
\toprule
\#Bulky Object & 1 & 2 & 3 & 4 & 5 & 6 & 7 & 9 & 10     \\
\midrule
Discretization & 2.13 & 4.31 & 6.67 & 5.96 & 9.33 & 6.18 & 12.44 & 21.54 & 14.43 \\
Exchange       & 0.11 & 0.06 & 0.11 & 0.14 & 0.22 & 0.31 & 0.84 & 1.93 & 0.40    \\
\toprule
\#Thin Object  & 1 & 2 & 3 & 4 & 5 & 6 & 7 & 9 & 10     \\
\midrule
Discretization & 3.90 & 8.43 & 4.74 & 22.87 & 24.10 & 9.71 & 23.43 & 10.87 & 11.94    \\
Exchange       & -    & -    & -    & -     & -     & -    & -     & -     & -        \\
\bottomrule
\end{tabular}
\caption{\label{table:settleTime} Computational time in seconds for settling 10 objects using our discretization method or the exchange method. We use ``-'' to indicate failed runs.}
\end{table}
Our first comparative benchmark involves object settling, where we drop non-convex objects into a box and use both methods to compute the force equilibrium poses. This can be achieved by setting the objective function to the gravitational potential energy and optimizing a single pose for all the objects. As compared with trajectory optimization, object settling is a much easier task, since no temporal subdivision or sampling is needed. We consider two modes of object settling: the first one-by-one mode sequentially optimizes one object's pose per-run, assuming all previously optimized objects stay still; the second joint mode optimizes all the objects' poses in a single run. The one-by-one mode is faster to compute, but cannot find accurate force equilibrium poses. The joint mode can find accurate poses, but takes considerably more iterations to converge. In our first experiment, we drop 18 bulky objects into a box and the results are shown in~\prettyref{fig:settle}a and the average computational time for the first 10 objects is summarized in~\prettyref{table:settleTime}. The exchange method is more than 10$\times$ faster than our method. This is because our method needs to consider all potential contacts between all pairs of triangles during each iteration. In our second experiment, we drop 18 eyeglasses into the box as illustrated in~\prettyref{fig:settle}c. Our method can still find force equilibrium poses, but the exchange method fails to find a feasible solution, because the eyeglasses have thin geometries with no well-defined penetration depth. Finally, we run our method in joint mode for both examples, where our method takes 19 minutes for the bulky objects and 20 minutes for the thin objects. As illustrated in~\prettyref{fig:settle}bd, the joint mode computes poses with much lower gravitational potential energy.

\begin{figure}[ht]
\centering
\begin{tabular}{cc}
\frame{\includegraphics[width=.45\linewidth,trim=18cm 15cm 2cm 15cm,clip]{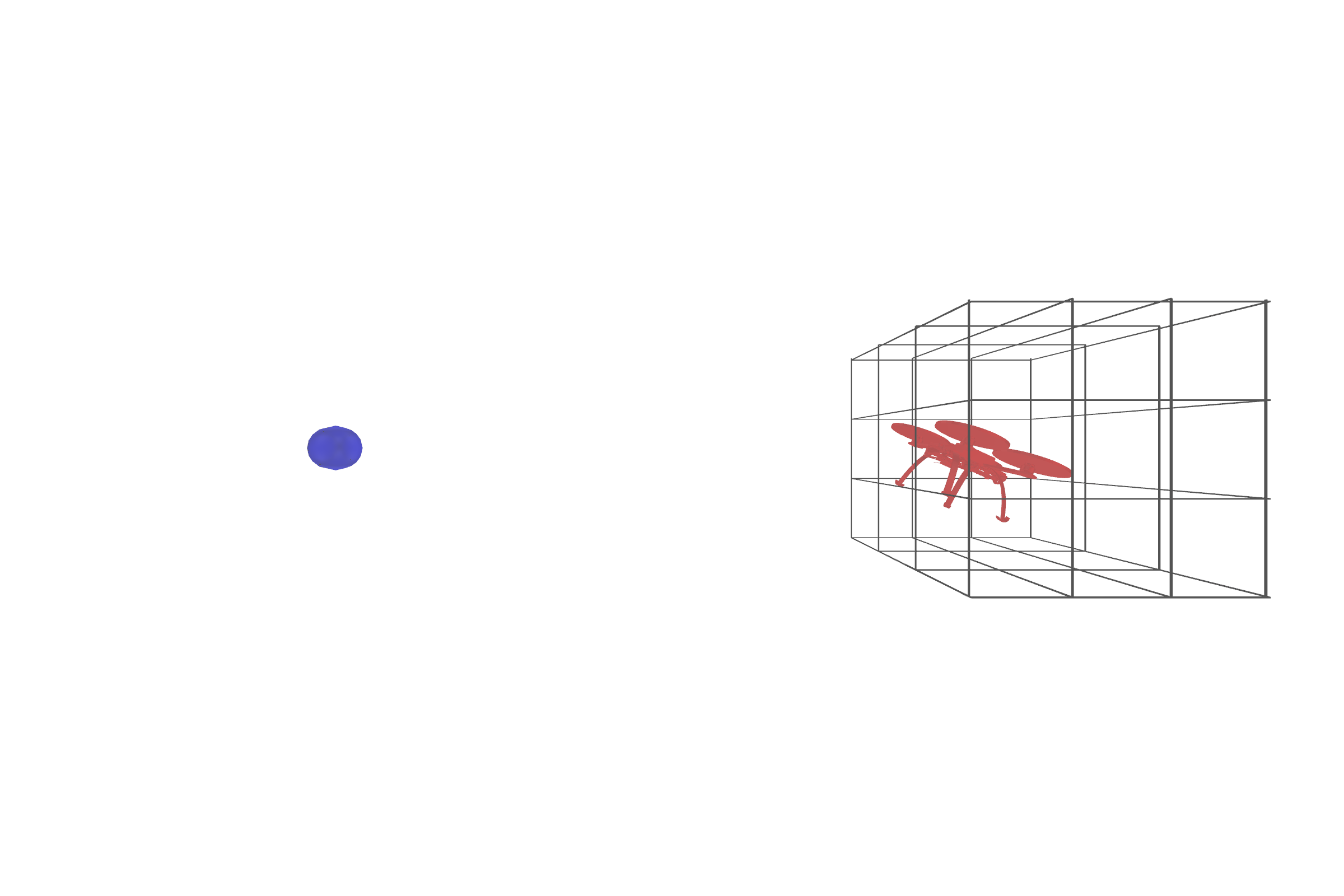}}
\put(-70,5){(a)}&
\frame{\includegraphics[width=.45\linewidth,trim=18cm 15cm 2cm 15cm,clip]{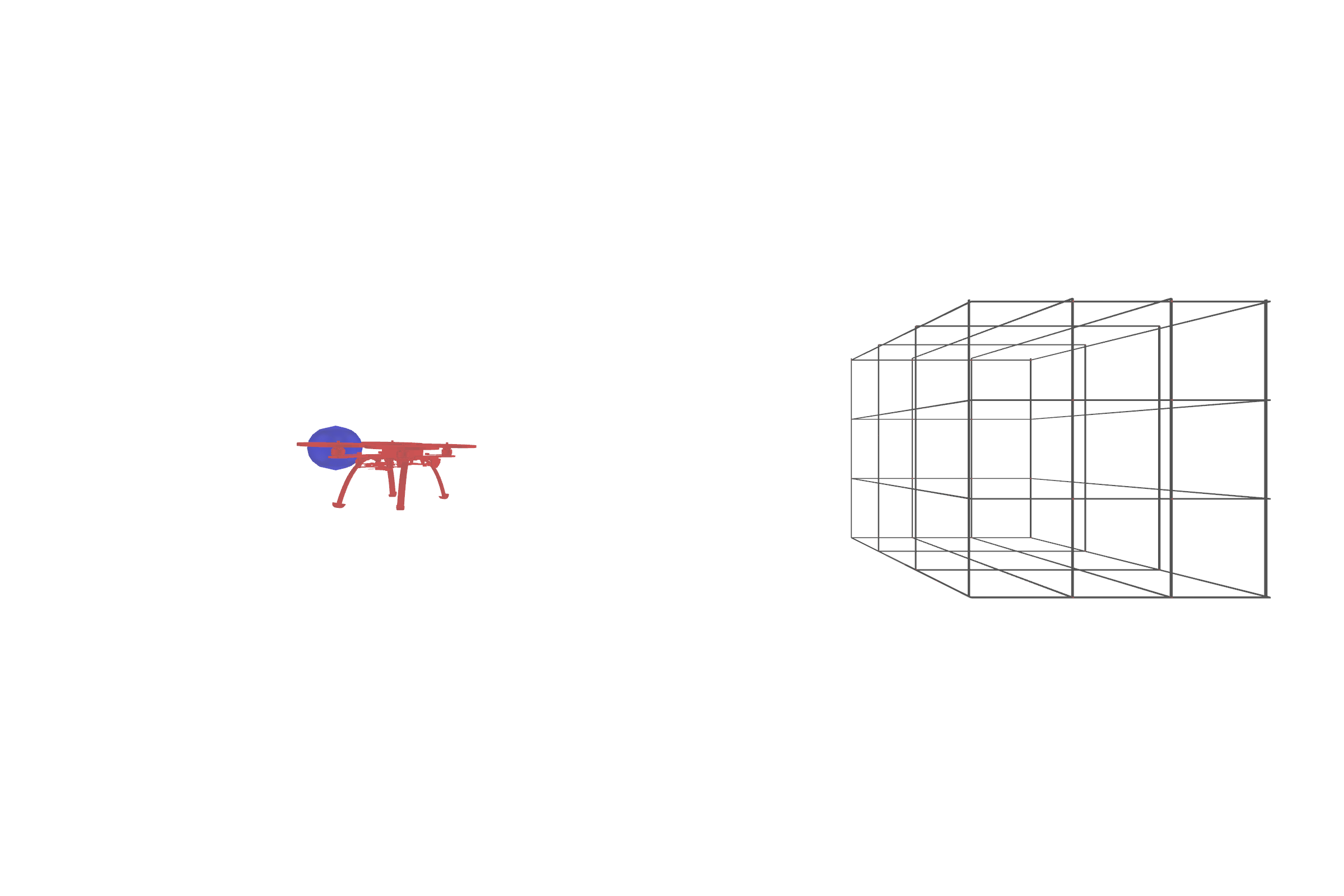}}
\put(-70,5){(b)}
\end{tabular}
\caption{\label{fig:failure} A UAV is trapped inside a cage but trying to fly outside to the blue target point. We show final location of the UAV computed by our discretization method (a) and the exchange method (b).}
\end{figure}
Our second comparative benchmark uses a toy example illustrated in~\prettyref{fig:failure}, where a UAV is trapped in a cage consisting of a row of metal bars. The UAV is trying to fly to the target point outside. Our method would get the UAV close to the target point but still trapped inside the cage, while the exchange method will erroneously get the UAV outside the cage however small $\epsilon$ is. This seemingly surprising result is due to the fact that the exchange method ultimately detects penetrations at discrete time instances with a resolution of $\epsilon$. However, the optimizer is allowed to arbitrarily accelerate the UAV to a point where it can fly out of the cage within $\epsilon$, and the collision constraint will be missed.

Finally, we run the two algorithms on the first and third benchmarks (\prettyref{fig:examples}ac), which involve only a single target position for the robots. (The other two benchmarks involve a grid of target positions making it too costly to compute using the exchange method.) The exchange method succeeds in both benchmarks. On the first benchmark, the exchange method takes 33.09 minutes to finish the computation after 123 index set updates and NLP solves, while our method only takes 5.47 minutes. Similarly, on the third benchmark, the exchange method takes 699.23 minutes to finish the computation after 1804 index set updates and NLP solves, while our method only takes 30.02 minutes. The performance advantage of our method is for two reasons: First, our method does not require the collision to be detected at the finest resolution ($\epsilon=10^{-3}$ in the exchange method). Instead, we only need a subdivision that is sufficient to find descendent directions. Second, our method does not require the NLP to be solved exactly after each update of the energy function. We only run one iteration of Newton's method per round of subdivision.

\subsection{Comparison with Sampling-Based Method}
In contrast to our locally optimal guarantee, sampling-based methods provide a stronger asymptotic global optimality guarantee, so we use the open-source sampling-based algorithm implementation~\cite{sucan2012the-open-motion-planning-library} as a groundtruth and set their low-level discrete collision checker to use a small sampling interval of $\epsilon=10^{-3}$. In our first comparison, we run our method and RRT-Star on the first benchmark (\prettyref{fig:examples}a) with a timeout of 60 minutes. For both methods, we set the objective function to be the weighted combination of the configuration-space trajectory length and the Cartesian-space distance to the target end-effector position. Throughout the optimizations, we save the best trajectory every 10 seconds and profile their trajectory length and distance to target in~\prettyref{fig:RRTStar}. Our method converges much faster to a locally optimal solution, while RRT-Star takes a long computational time to search for better trajectories in the 12-dimensional configuration space, making little progress. We have also tested RRT-Star on the third benchmark (\prettyref{fig:examples}c) with an even higher 54-dimensional configuration space, but it fails to compute a meaningful trajectory within 60 minutes. This is due to the exponential complexity of the zeroth-order sampling-based algorithms.

\begin{figure}[ht]
\centering
\begin{tabular}{cc}
\includegraphics[width=.45\linewidth]{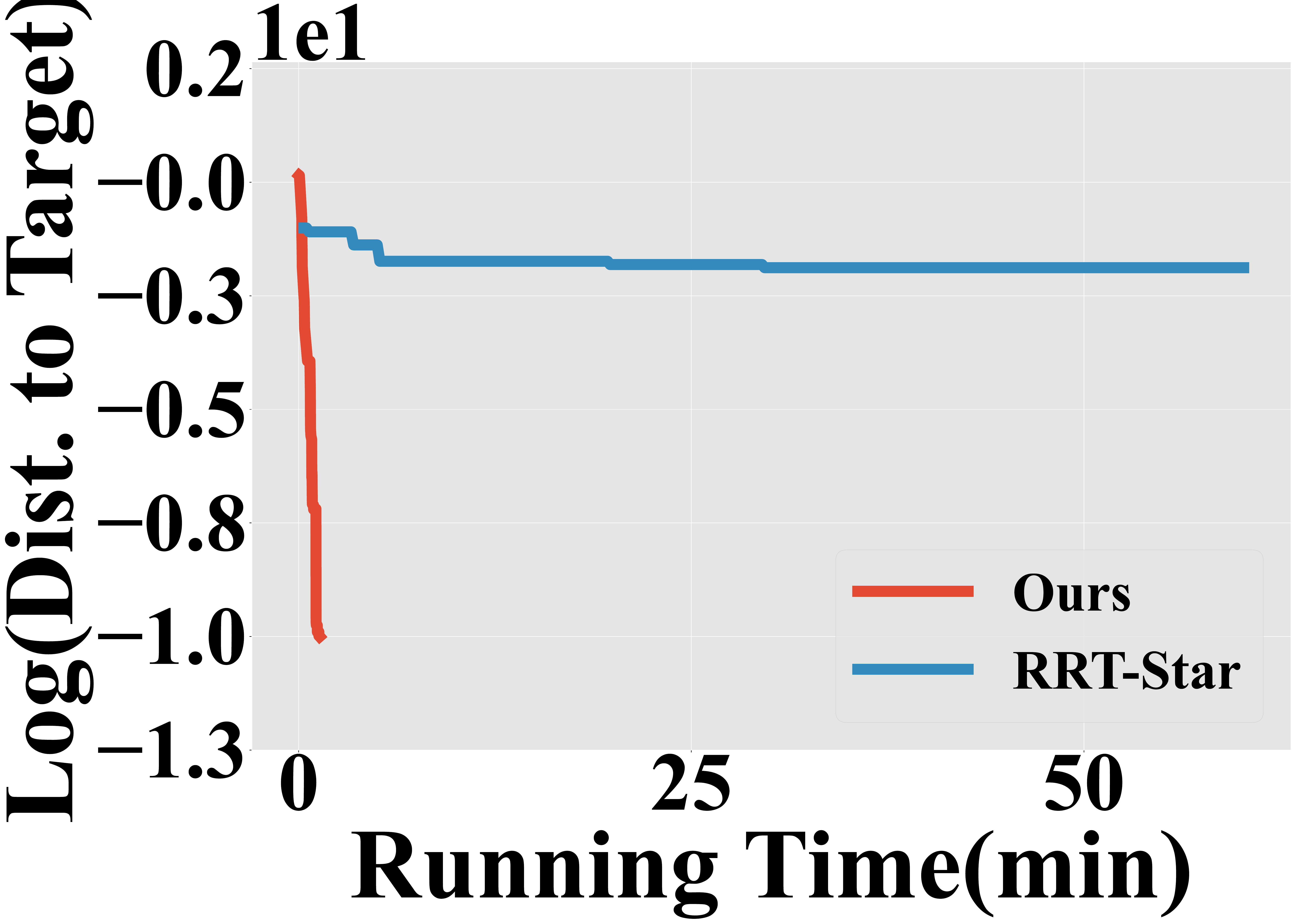}
\put(-5,3){(a)} &
\includegraphics[width=.45\linewidth]{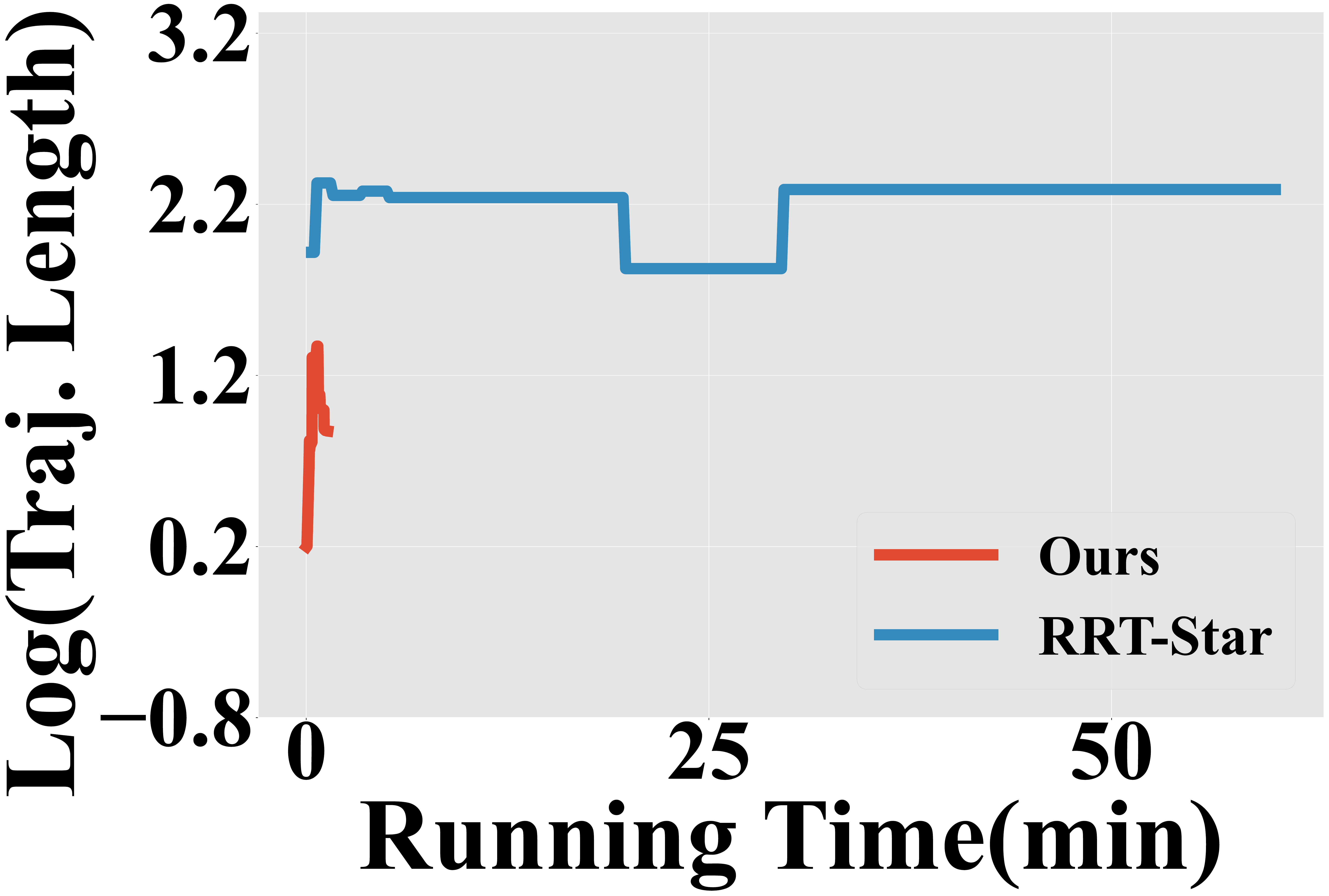}
\put(-5,3){(b)}
\end{tabular}
\caption{\label{fig:RRTStar} We profile the Cartesian-space distance to target (a) and the configuration-space trajectory length (b) over a trajectory optimization procedure using our algorithm and RRT-Star, both in log-scale.}
\end{figure}
In our second comparison, we run PRM on the second benchmark (\prettyref{fig:examples}b) with a timeout of 20 minutes, which is already longer than the computational time taken by our method to reach all target positions, and we consider a target position reached if the end-effector is within 7 centimeters from it. In this 6-dimensional configuration space, PRM succeeds in reaching 46/56 target positions, while our method reaches 39/56 target positions. Our lower success rate is due to our locally optimal nature. For a higher success rate, the sampling-based algorithms can be combined with our local method to provide a high-quality initial guess, e.g., as done in~\cite{choudhury2016regionally}, but this is beyond the scope of this research. We further test PRM on the fourth benchmark (\prettyref{fig:examples}d) with a timeout of 240 minutes, which is again longer than the total computational time of our method. In this 11-dimensional configuration space, PRM can only reach 17/114 target positions, while our method reaches 88/114 target positions, highlighting the benefits of our first-order algorithms.

\begin{table}[ht]
\centering
\setlength{\tabcolsep}{3pt}
\begin{tabular}{lcccccc}
\toprule
Interval $\epsilon$ & 0.1 & 0.05 & 0.025 & 0.0125 & 0.01 \\
\midrule
\#UAV-Escape & 10/10 & 7/10 & 2/10 & 0/10 & 0/10 \\
\bottomrule
\end{tabular}
\caption{\label{table:UAVEscape} The behavior of UAV planned by RRT-Star under different sample intervals $\epsilon$.}
\end{table}
Finally, we compare the two methods on the toy example in~\prettyref{fig:failure} to highlight the sensitivity of RRT-Star to the sampling interval of the collision checker. Again, we set the timeout to 30 minutes for RRT-Star. We use different sampling interval $\epsilon$ and run RRT-Star for 10 times under each $\epsilon$. The number of times when UAV can incorrectly escape the cage is summarized in~\prettyref{table:UAVEscape}. The UAV is correctly trapped in the cage under sufficiently small $\epsilon$, while it escapes under large $\epsilon$. For intermediary values of $\epsilon$, the behavior of UAV depends on the random sample locations. In contrast, our method can guarantee that the UAV is correctly trapped inside the cage without tuning $\epsilon$.
\section{Conclusion}
We propose a provably feasible algorithm for collision-free trajectory generation of articulated robots. We formulate the underlying SIP problem and solve it using a novel feasible discretization method. Our method divides the temporal domain into discrete intervals and chooses one representative constraint for each interval, reducing the SIP to an NLP. We further propose a conservative motion bound that ensures the original SIP constraint is satisfied. Finally, we establish theoretical convergence guarantee and propose practical implementations for articulated robots. Our results show that our method can generate long-horizon trajectories for industrial robot arms within a couple minutes of computation.

Our method pertains several limitations that we consider to address as future work. On the downside, our method requires a large number of subdivisions to approach a feasible and nearly optimal solution. This is partly due to an overly conservative Lipschitz constant estimation and a linear motion bound~\prettyref{lem:bound}. In the future, we plan to reduce the number of subdivisions and improve the computational speed by exploring high-order Taylor models~\cite{10.1145/1276377.1276396}. As a minor limitation, our method requires a customized line-search scheme and cannot use off-the-shelf NLP solvers, which potentially increase the implementation complexity. \revised{Finally, our method does not account for equality constraints, which is useful for modeling the dynamics of the robot. Additional equality constraints can be incorporated by combining our method with merit-function-based techniques~\cite{bertsekas1997nonlinear}, which is an essential avenue of future work.}
\printbibliography
\section{\label{sec:termination} Finite Termination of \prettyref{alg:FIPM}}
We show that \prettyref{alg:FIPM} would terminate after finitely many iterations. We will omit the parameter of a function whenever there can be no confusion. Our main idea is to compare the following two terms:
\begin{align*}
\tilde{\mathcal{P}}\triangleq\sum_{ijkl}(T_1^l-T_0^l)\mathcal{P}_{ijkl}\quad
\bar{\mathcal{P}}\triangleq\sum_{ijkl}\int_{T_0^l}^{T_1^l}\mathcal{P}_{ijk}(t)dt,
\end{align*}
where we use a shorthand notation $\tilde{\mathcal{P}}$ for the penalty function part of \prettyref{eq:penalty}. Conceptually, $\tilde{\mathcal{P}}$ approximates $\bar{\mathcal{P}}$ in the sense of Riemann sum and the approximation error would reduce as more subdivisions are performed. However, the approximation error will not approach zero because our subdivision is adaptive. Therefore, we need a new tool to analyze the difference between $\tilde{\mathcal{P}}$ and $\bar{\mathcal{P}}$. To this end, we introduce the following hybrid penalty function with a variable $\epsilon_2$ controlling the level of hybridization:
\begin{align*}
\hat{\mathcal{P}}(\epsilon_2)=\sum_{ijkl}\begin{cases}
(T_1^l-T_0^l)\mathcal{P}_{ijkl}\quad&T_1^l-T_0^l\geq\epsilon_2\\
\int_{T_0^l}^{T_1^l}\mathcal{P}_{ijk}(t)dt\quad&T_1^l-T_0^l<\epsilon_2,
\end{cases}
\end{align*}
where we use the integral form when a temporal interval is shorter than $\epsilon_2$, and use the surrogate constraint otherwise. An important property of $\hat{\mathcal{P}}(\epsilon_2)$ is that it is invariant to subdivision after finitely many iterations:
\begin{lemma}
\label{lem:invariance}
Given fixed $\epsilon_2$, and after finitely many times of subdivision, $\hat{\mathcal{P}}(\epsilon_2)$ becomes invariant to further subdivision.
\end{lemma}
\begin{proof}
Since each subdivision would reduce a time interval by a factor of $1/2$, it takes finitely many subdivisions to reduce a time interval to satisfy: $T_1^l-T_0^l<\epsilon_2$. Therefore, after finitely many subdivision operators, a time interval must satisfy one of two cases: (Case I) No more subdivisions are applied to it, making it invariant to further subdivisions; (Case II) The time interval $T_1^l-T_0^l<\epsilon_2$ and infinitely many subdivisions will be applied, but the integral is invariant to subdivision.
\end{proof}
Next, we show that the difference between $\hat{\mathcal{P}}$ and $\tilde{\mathcal{P}}$ is controllable via $\epsilon_2$. The following result bound their differences:
\begin{lemma}
\label{lem:valueBound}
Taking \prettyref{ass:Spatial}, \ref{ass:Bounded}, \ref{ass:Barrier}, and assuming $\theta$ is generated by some iteration of \prettyref{alg:FIPM}, we have: 
\begin{align*}
\left|\hat{\mathcal{P}}(\epsilon_2)-\tilde{\mathcal{P}}\right|=\textbf{O}(\epsilon_2^{1-5\eta}),
\end{align*}
for arbitrarily small fixed $\epsilon_2$.
\end{lemma}
\begin{proof}
We use the shorthand notation $\sum_{ijkl}^{\Delta T<\epsilon_2}$ to denote a summation over intervals $T_1^l-T_0^l<\epsilon_2$, and the following abbreviations are used:
\begin{align*}
d_t\triangleq&\dist\left(b_{ij}\left(t,\theta\right),o_k\right)-d_0\\
d_m\triangleq&\dist\left(b_{ij}\left(\frac{T_0^l+T_1^l}{2},\theta\right),o_k\right)-d_0.
\end{align*}
Since $\theta$ is generated by line search, we have $\theta$ passes the safety check, leading to the following result:
\begin{align*}
&\left|\mathcal{P}_{ijk}(t)-\mathcal{P}_{ijkl}\right|
=\left|\int_{d_t}^{d_m}\FDD{\mathcal{P}(x)}{x}dx\right|\\
\leq&\left|\FDD{\mathcal{P}(x)}{x}\Bigg|_{L_2(T_1^l-T_0^l)^\eta}\right|\left|d_m-d_t\right|\\
\leq&L_1\left|\FDD{\mathcal{P}(x)}{x}\Bigg|_{L_2(T_1^l-T_0^l)^\eta}\right|\left|t-\frac{T_0^l+T_1^l}{2}\right|.
\end{align*}
The second inequality above is due to the safety check condition and monotonicity of $\mathcal{P},|\nabla_x\mathcal{P}|$. The third inequality above is due to \prettyref{lem:bound}. The result in our lemma is derived immediately as follows:
\begin{align*}
&\left|\hat{\mathcal{P}}(\epsilon_2)-\tilde{\mathcal{P}}\right|
\leq\sum_{ijkl}^{\Delta T<\epsilon_2}\int_{T_0^l}^{T_1^l}
\left|\mathcal{P}_{ijk}(t)-\mathcal{P}_{ijkl}\right|dt\\
\leq&\sum_{ijkl}^{\Delta T<\epsilon_2}
L_1\left|\FDD{\mathcal{P}(x)}{x}\Bigg|_{L_2(T_1^l-T_0^l)^\eta}\right|
\int_{T_0^l}^{T_1^l}\left|t-\frac{T_0^l+T_1^l}{2}\right|dt\\
\leq&\sum_{ijkl}^{\Delta T<\epsilon_2}
L_1\frac{(T_1^l-T_0^l)^2}{4}\left|\FDD{\mathcal{P}(x)}{x}\Bigg|_{L_2(T_1^l-T_0^l)^\eta}\right|\\
\leq&\sum_{ijkl}\frac{T+\epsilon_2}{\epsilon_2}L_1
\frac{(T_1^l-T_0^l)^2}{4}\left|\FDD{\mathcal{P}(x)}{x}\Bigg|_{L_2(T_1^l-T_0^l)^\eta}\right|\\
=&\sum_{ijkl}\frac{T+\epsilon_2}{\epsilon_2}L_1(T_1^l-T_0^l)\Gamma(T_1^l-T_0^l)\\
&\Gamma(T_1^l-T_0^l)\triangleq\frac{T_1^l-T_0^l}{4}\left|\FDD{\mathcal{P}(x)}{x}\Bigg|_{L_2(T_1^l-T_0^l)^\eta}\right|.
\end{align*}
The last inequality above is by the assumption that the entire temporal domain $[0,T]$ is subdivided into intervals of length smaller than $\epsilon_2$. It can be shown by direct verification that by choosing $\eta<1/5$, $\Gamma(T_1^l-T_0^l)=\mathbf{O}(\epsilon_2^{1-5\eta})$ as $T_1^l-T_0^l\to0$ and the lemma is proved.
\end{proof}
In a similar fashion to \prettyref{lem:valueBound}, we can bound the difference in gradient:
\begin{lemma}
\label{lem:gradientBound}
Taking \prettyref{ass:Spatial}, \ref{ass:Bounded}, \ref{ass:Barrier}, and assuming $\theta$ is generated by some iteration of \prettyref{alg:FIPM}, we have: 
\begin{align*}
\left\|\nabla_\theta\hat{\mathcal{P}}(\epsilon_2)-\nabla_\theta\tilde{\mathcal{P}}\right\|=\textbf{O}(\epsilon_2^{1-6\eta}),
\end{align*}
for arbitrarily small fixed $\epsilon_2$.
\end{lemma}
\begin{proof}
Again, we use the shorthand notations: $\sum_{ijkl}^{\Delta T<\epsilon_2}$, $d_t$, and $d_m$. We begin by bounding the error of the integrand:
\begin{align*}
&\left\|\nabla_\theta\mathcal{P}_{ijk}(t)-\nabla_\theta\mathcal{P}_{ijkl}\right\|\\
=&\left\|\FDD{\mathcal{P}(d_t)}{d_t}\nabla_\theta d_t-\FDD{\mathcal{P}(d_m)}{d_m}\nabla_\theta d_m\right\|\\
\leq&\left|\FDD{\mathcal{P}(d_t)}{d_t}\right|\|\nabla_\theta d_t-\nabla_\theta d_m\|+
\left|\FDD{\mathcal{P}(d_t)}{d_t}-\FDD{\mathcal{P}(d_m)}{d_m}\right|\|\nabla_\theta d_m\|,
\end{align*}
which is due to triangle inequality. There are two terms in the last equation to be bounded. To bound the first term, we use a similar argument as \prettyref{lem:bound}. Under \prettyref{ass:Bounded}, there must exist some constant $L_3$ such that:
\begin{align*}
\|\nabla_\theta d_t-\nabla_\theta d_m\|\leq L_3\left|t-\frac{T_0^l+T_1^l}{2}\right|.
\end{align*}
Since $\theta$ passes the safety check, we further have:
\small
\begin{align*}
\left|\FDD{\mathcal{P}(d_t)}{d_t}\right|\|\nabla_\theta d_t-\nabla_\theta d_m\|
\leq L_3\left|\FDD{\mathcal{P}(x)}{x}\Bigg|_{L_2(T_1^l-T_0^l)^\eta}\right|
\left|t-\frac{T_0^l+T_1^l}{2}\right|.
\end{align*}
\normalsize
To bound the second term, we note that $\|\nabla_\theta d_m\|\leq L_4$ for some $L_4$ because its domain is bounded. By the mean value theorem, we have:
\begin{align*}
&\left|\FDD{\mathcal{P}(d_t)}{d_t}-\FDD{\mathcal{P}(d_m)}{d_m}\right|\|\nabla_\theta d_m\|
\leq L_4\left|\int_{d_t}^{d_m}\FPPT{\mathcal{P}(x)}{x}dx\right|\\
\leq&L_4\left|\FPPT{\mathcal{P}(x)}{x}\Bigg|_{L_2(T_1^l-T_0^l)^\eta}\right|\left|t-\frac{T_0^l+T_1^l}{2}\right|,
\end{align*}
where we have used the safety check condition and monotonicity of $|\nabla_x^2\mathcal{P}|$. Putting everything together, we establish the result in our lemma as follows:
\begin{align*}
&\left\|\nabla_\theta\hat{\mathcal{P}}(\epsilon_2)-\nabla_\theta\tilde{\mathcal{P}}\right\|\leq
\sum_{ijkl}^{\Delta T<\epsilon_2}\\
&L_3(T_1^l-T_0^l)\Gamma(T_1^l-T_0^l)+L_4(T_1^l-T_0^l)\Gamma'(T_1^l-T_0^l)\\
&\Gamma'(T_1^l-T_0^l)\triangleq\frac{T_1^l-T_0^l}{4}\left|\FPPT{\mathcal{P}(x)}{x}\Bigg|_{L_2(T_1^l-T_0^l)^\eta}\right|.
\end{align*}
It can be shown that $\Gamma'$ is the dominating term and, by direct verification, we have $\Gamma'(T_1^l-T_0^l)=\mathbf{O}(\epsilon_2^{1-6\eta})$ as $T_1^l-T_0^l\to0$, which proves our lemma.
\end{proof}

\subsection{Finite Termination of \prettyref{alg:search}}
We are now ready to show the finite termination of the line search \prettyref{alg:search}. If \prettyref{alg:search} does not terminate, it must make infinitely many calls to the subdivision function. Otherwise, suppose only finitely many calls to the subdivision is used, \prettyref{alg:search} reduces to a standard line search for NLP after the last call, which is guaranteed to succeed. However, we show that using infinitely many subdivisions will contradict the finiteness of $\mathcal{E}(\theta)$.
\begin{lemma}
\label{lem:unbounded}
Taking \prettyref{ass:Spatial}, \ref{ass:Bounded}, \ref{ass:Barrier}, if \prettyref{alg:search} makes infinitely many calls to subdivision, then $\theta$ is unsafe. In other words, $\theta$ cannot pass the safety check~\prettyref{alg:safety} under any finite spatial-temporal subdivision.
\end{lemma}
\begin{proof}
Suppose $\theta$ is safe and~\prettyref{alg:search} is trying to update $\theta$ to $\theta'=\theta+d\alpha$. Such update must fail because only finitely many subdivisions are needed otherwise. Further, there must be some interval $[T_0^l,T_1^l]$ that requires infinitely many subdivisions. As a result, given any fixed $\epsilon_3$ and $\epsilon_4$, there must be some unsafe interval $[\bar{T}_0^l,\bar{T}_1^l]\subset[T_0^l,T_1^l]$ such that:
\begin{align*}
\bar{T}_1^l-\bar{T}_0^l\leq\epsilon_3\quad
\alpha\leq\epsilon_4.
\end{align*}
We use the following shorthand notation:
\begin{align*}
\bar{\mathcal{P}}_{ijkl}\triangleq&
\mathcal{P}\left(\dist\left(b_{ij}\left(\frac{\bar{T}_0^l+\bar{T}_1^l}{2},\theta\right),o_k\right)-d_0\right)\\
\mathcal{P}_{ijkl}'\triangleq&
\mathcal{P}\left(\dist\left(b_{ij}\left(\frac{\bar{T}_0^l+\bar{T}_1^l}{2},\theta'\right),o_k\right)-d_0\right)\\
d_m\triangleq&\dist\left(b_{ij}\left(\frac{\bar{T}_0^l+\bar{T}_1^l}{2},\theta\right),o_k\right)-d_0\\
d_m'\triangleq&\dist\left(b_{ij}\left(\frac{\bar{T}_0^l+\bar{T}_1^l}{2},\theta'\right),o_k\right)-d_0.
\end{align*}
Since the interval is unsafe, we have:
\begin{align*}
\mathcal{P}_{ijkl}'\geq\mathcal{P}(\psi(\bar{T}_1^l-\bar{T}_0^l)).
\end{align*}
Finally, we can bound the difference between penalty functions evaluated at $\theta$ and that at $\theta'$ using mean value theorem:
\begin{align*}
&\left|\mathcal{P}_{ijkl}'-\bar{\mathcal{P}}_{ijkl}\right|
=\left|\int_{d_m}^{d_m'}\FDD{\mathcal{P}(x)}{x}dx\right|\\
\leq&\fmax{x\in[d_m,d_m']}\left|\FDD{\mathcal{P}(x)}{x}\right|\left|d_m-d_m'\right|\\
\leq&L_4\fmax{x\in[d_m-L_4\left\|d\right\|\epsilon_4,d_m+L_4\left\|d\right\|\epsilon_4]}\left|\FDD{\mathcal{P}(x)}{x}\right|\left\|d\right\|\epsilon_4.
\end{align*}
The above result implies that the difference between the two penalty functions is controllable via $\epsilon_4$. Since both $\epsilon_3$ and $\epsilon_4$ are arbitrary and independent, we can first choose small $\epsilon_4$ such that:
\begin{align*}
&\bar{\mathcal{P}}_{ijkl}=\mathbf{\Theta}(\mathcal{P}_{ijkl}')\geq\mathbf{\Theta}(\mathcal{P}(\psi(\bar{T}_1^l-\bar{T}_0^l)))\geq\mathbf{\Theta}(\mathcal{P}(\psi(\epsilon_3))),
\end{align*}
and then choose small $\epsilon_3$ to make $\bar{\mathcal{P}}_{ijkl}$ arbitrarily large by~\prettyref{lem:infty}. But for $\theta$ to be safe, we need:
\begin{align*}
\bar{\mathcal{P}}_{ijkl}\leq\mathcal{P}(L_2(T_1^l-T_0^l)^\eta),
\end{align*}
 leading to a contradiction, so $\theta$ cannot be safe.
\end{proof}
\begin{corollary}
\label{cor:terminationInner}
\prettyref{alg:search} makes finitely many calls to subdivision, i.e., terminates finitely.
\end{corollary}
\begin{proof}
\prettyref{alg:FIPM} requires the initial $\theta$ to be feasible, so the initial $\theta$ is safe. Each iteration of \prettyref{alg:FIPM} generates feasible iterations by~\prettyref{thm:feasible}. If infinitely many subdivisions are used, then~\prettyref{lem:unbounded} implies that some $\theta$ is unsafe, which is a contradiction.
\end{proof}

\subsection{Finite Termination of \prettyref{alg:FIPM}}
After showing the finite termination of \prettyref{lem:unbounded}, we move on to show the finite termination of main \prettyref{alg:FIPM}. Our main idea is to compare $\tilde{\mathcal{P}}$ and $\hat{\mathcal{P}}$ and bound their difference. We first show that: $\hat{\mathcal{P}}$ is unbounded if infinite number of subdivisions are needed:
\begin{lemma}
\label{lem:unboundedEpsilon}
Taking \prettyref{ass:Spatial}, \ref{ass:Bounded}, \ref{ass:Barrier}, for any fixed $\epsilon_5$, if \prettyref{alg:search} makes infinitely many calls to subdivision, then either $\hat{\mathcal{P}}(\epsilon_5)$ is unbounded or $\theta$ is unsafe.
\end{lemma}
\begin{proof}
Following the same argument as~\prettyref{lem:unbounded}, there must be unsafe interval $[\bar{T}_0^l,\bar{T}_1^l]\subset[T_0^l,T_1^l]$ with $\mathcal{P}_{ijkl}=\mathbf{\Theta}(\mathcal{P}'_{ijkl})\geq\mathbf{\Theta}(\mathcal{P}(\psi(\epsilon_3))$ for any fixed $\epsilon_3$. Since the domain is compact, there must be some $t\in[T_0^l,T_1^l]$ such that $\mathcal{P}_{ijk}(t)=\infty$. There are two cases for the interval $[T_0^l,T_1^l]$: (Case I) If $T_1^l-T_0^l<\epsilon_5$, then $\hat{\mathcal{P}}(\epsilon_5)$ is using the integral formula for the interval and $\mathcal{P}$ is unbounded by~\prettyref{lem:infty}. (Case II) If $T_1^l-T_0^l\geq\epsilon_5$, then $\theta$ is unsafe by~\prettyref{lem:unbounded}.
\end{proof}
Our final proof uses the Wolfe's condition to derive a contradiction if infinite number of subdivisions are needed. Specifically, we will show that the search direction is descendent if $\tilde{P}$ is replaced by $\hat{P}$ for some small $\epsilon_2$. The following argument assumes $d=d^{(1)}$ and the case with $d=d^{(2)}$ follows an almost identical argument.
\begin{proof}[Proof of \prettyref{thm:termination}]
Suppose otherwise, we have $\|d\|_\infty\geq\epsilon_d$ because the algorithm terminates immediately otherwise. Due to the equivalence of metrics, we have $\|d\|\geq\epsilon_6$ for some $\epsilon_6$. We introduce the following shorthand notation:
\begin{align*}
\hat{\mathcal{E}}(\theta,\epsilon_2)=\mathcal{O}(\theta)+\mu\hat{\mathcal{P}}(\epsilon_2).
\end{align*}
We consider an iteration of \prettyref{alg:FIPM} that updates from $\theta$ to $\theta'$. Since the first Wolfe's condition holds, we have:
\begin{align*}
\mathcal{E}(\theta')\leq\mathcal{E}(\theta)-c\|\nabla_\theta\mathcal{E}(\theta)\|^2\alpha
\leq\mathcal{E}(\theta)-c\|\nabla_\theta\mathcal{E}(\theta)\|\alpha\epsilon_6.
\end{align*}
The corresponding change in $\hat{\mathcal{E}}(\theta,\epsilon_2)$ can be bounded as follows:
\begin{align*}
&\hat{\mathcal{E}}(\theta',\epsilon_2)=\hat{\mathcal{E}}(\theta,\epsilon_2)+
\int_{\theta}^{\theta'}\left<\nabla_\theta\hat{\mathcal{E}}(\theta,\epsilon_2),d\theta\right>\\
\leq&\hat{\mathcal{E}}(\theta,\epsilon_2)+
\int_{\theta}^{\theta'}\left<\nabla_\theta\hat{\mathcal{E}}(\theta,\epsilon_2)
-\nabla_\theta\mathcal{E}(\theta)
+\nabla_\theta\mathcal{E}(\theta),d\theta\right>\\
\leq&\hat{\mathcal{E}}(\theta,\epsilon_2)+
\int_{\theta}^{\theta'}\|\nabla_\theta\hat{\mathcal{E}}(\theta,\epsilon_2)
-\nabla_\theta\mathcal{E}(\theta)\|\|d\theta\|+\mathcal{E}(\theta')-\mathcal{E}(\theta)\\
\leq&\hat{\mathcal{E}}(\theta,\epsilon_2)+\textbf{O}(\epsilon_2^{1-6\eta})\|d^{(1)}\|\alpha-c\|\nabla_\theta\mathcal{E}(\theta)\|\alpha\epsilon_6\\
=&\hat{\mathcal{E}}(\theta,\epsilon_2)+\|\nabla_\theta\mathcal{E}(\theta)\|\alpha(\textbf{O}(\epsilon_2^{1-6\eta})-c\epsilon_6).
\end{align*}
As long as $\eta<1/6$, we can choose sufficiently small $\epsilon_2$ such that $\hat{\mathcal{E}}(\theta',\epsilon_2)<\hat{\mathcal{E}}(\theta,\epsilon_2)$. Since we assume there are infinitely many subdivisions and \prettyref{cor:terminationInner} shows that line search \prettyref{alg:search} will always terminate finitely, we conclude that~\prettyref{alg:FIPM} will generate an infinite sequence $\theta$ of decreasing $\hat{\mathcal{E}}(\theta,\epsilon_2)$ for sufficiently small $\epsilon_2$. Further, each $\theta$ is safe and each $\hat{\mathcal{E}}(\theta,\epsilon_2)$ is finite by the motion bound. But these properties contradict \prettyref{lem:unboundedEpsilon}.
\end{proof}
\section{\label{sec:optimality}Using \prettyref{alg:FIPM} as SIP Solver}
We show that \prettyref{alg:FIPM} is indeed a solver of the SIP problem \prettyref{eq:prob}. To this end, we consider running \prettyref{alg:FIPM} for an infinite number of iterations and we use superscript to denote iteration number. At the $k$th iteration, we use $\mu=\mu^k, \epsilon_d=\epsilon_d^k$ and we assume the sequences $\{\mu^k\}$ and $\{\epsilon_d^k\}$ are both null sequences. This will generate a sequence of solutions $\{\theta^k\}$ and we consider one of its convergent subsequence also denoted as $\{\theta^k\}\to\theta^0$. We consider the first-order optimality condition at $\theta^0$:
\begin{define}
If $\theta^0$ satifies the first-order optimality condition, then for each direction $D_\theta$ such that:
\begin{align*}
\left<D_\theta,\nabla_\theta\dist(b_{ij}(t,\theta^0),o_k)\right>\geq0\quad\forall \dist(b_{ij}(t,\theta^0),o_k)=0,
\end{align*}
we have $\left<D_\theta,\nabla_\theta\mathcal{O}(\theta^0)\right>\geq0$.
\end{define}
We further assume the following generalized Mangasarian-Fromovitz constraint qualification (GMFCQ) holds at $\theta^0$:
\begin{assume}
\label{ass:GMFCQ}
There exists some direction $D_\theta$ and positive $\epsilon_7$ such that:
\begin{align*}
\left<D_\theta,\nabla_\theta\dist(b_{ij}(t,\theta^0),o_k)\right>\geq\epsilon_7\quad\forall \dist(b_{ij}(t,\theta^0),o_k)=0.
\end{align*}
\end{assume}
MFCQ is a standard assumption establishing connection between the first-order optimality condition of NLP and the gradient of the Lagrangian function. Our generalized version of MFCQ requires a positive constant $\epsilon_7$, which is essential for extending it to SIP. Note that GMFCQ is equivalent to standard MFCQ for NLP. We start by showing a standard consequence of assuming GMFCQ:
\begin{lemma}
\label{lem:D3}
Taking \prettyref{ass:GMFCQ}, if first-order optimality fails at a trajectory $\theta^0$, then we have a direction $D_\theta$ such that:
\begin{align*}
&\left<D_\theta,\nabla_\theta\mathcal{O}(\theta^0)\right><-\epsilon_8\\
&\left<D_\theta,\nabla_\theta\dist(b_{ij}(t,\theta^0),o_k)\right>>\epsilon_9\quad\forall \dist(b_{ij}(t,\theta^0),o_k)=0.
\end{align*}
\end{lemma}
\begin{proof}
Under our assumptions, there is a direction $D_\theta^1$ satisfying GMFCQ and another direction $D_\theta^2$ violating first-order optimality condition. We can then consider a third direction $D_\theta^3=D_\theta^1\epsilon_{10}+D_\theta^2$ where we have:
\begin{align*}
&\left<D_\theta^3,\nabla_\theta\mathcal{O}(\theta^0)\right>=\epsilon_{10}\left<D_\theta^1,\nabla_\theta\mathcal{O}(\theta^0)\right>-\epsilon_{11}\\
&\left<D_\theta^3,\nabla_\theta\dist(b_{ij}(t,\theta^0),o_k)\right>\geq\epsilon_{10}\epsilon_7,
\end{align*}
where $\epsilon_{11}$ is some positive constant. We can thus choose sufficiently small $\epsilon_{10}$ to make the righthand side of the first equation negative and the righthand side of the second one positive.
\end{proof}
Assuming a failure direction $D_\theta^3$ exists, we now start to show that the directional derivative of our objective function $\mathcal{E}$ along $D_\theta^3$ is bounded away from zero, which contradicts the fact that our gradient norm threshold $\{\epsilon_d^k\}$ is tending to zero. To bound the derivative near $\theta^0$, we need to classify $b_{ij}(t,\theta^0)$ into two categories: 1) its distance to $o$ is bounded away from zero; 2) its distance to $o$ is close to zero but $b_{ij}$ is moving away along $D_\theta^3$. This result is formalized below:
\begin{lemma}
\label{lem:category}
Taking \prettyref{ass:Spatial}, for $D_\theta^3$ stated in \prettyref{lem:D3} and each tuple of $\left<i,j,k\right>$, one of the following conditions holds:
\begin{align*}
&\dist(b_{ij}(t,\theta^0),o_k)>\epsilon_{12}\\
&\dist(b_{ij}(t,\theta^0),o_k)\leq\epsilon_{12}\land
\left<D_\theta^3,\nabla_\theta\dist(b_{ij}(t,\theta^0),o_k)\right>\geq\frac{\epsilon_9}{2},
\end{align*}
where $\epsilon_{12}$ is some positive constant.
\end{lemma}
\begin{proof}
Suppose otherwise, for arbitrarily small $\epsilon_{12}$, we can find some $i,j,k,t$ such that:
\begin{equation}
\begin{aligned}
\label{eq:distCategory}
&\dist(b_{ij}(t,\theta^0),o_k)\leq\epsilon_{12}\land\\
&\left<D_\theta^3,\nabla_\theta\dist(b_{ij}(t,\theta^0),o_k)\right><\frac{\epsilon_9}{2}.
\end{aligned}
\end{equation}
We can construct a sequence of $\{\left<i,j,k,t,\epsilon_{12}\right>\}$ with diminishing $\epsilon_{12}$ such that \prettyref{eq:distCategory} holds for each $\left<i,j,k,t,\epsilon_{12}\right>$ tuple. If the sequence is finite, then there must be some:
\begin{align*}
\dist(b_{ij}(t,\theta^0),o_k)=0\land
\left<D_\theta^3,\nabla_\theta\dist(b_{ij}(t,\theta^0),o_k)\right><\frac{\epsilon_9}{2},
\end{align*}
contradicting \prettyref{lem:D3}. If the sequence is infinite, then by \prettyref{ass:Spatial}, there is an infinite subsequence with $\left<i,j,k\right>$ being the same throughout the subsequence. We denote the subsequence as:
$\{\left<t,\epsilon_{12}\right>\}$, which tends to $\{\left<t^0,0\right>\}$. By the continuity of functions we have:
\begin{align*}
\dist(b_{ij}(t^0,\theta^0),o_k)=0\land
\left<D_\theta^3,\nabla_\theta\dist(b_{ij}(t^0,\theta^0),o_k)\right>\leq\frac{\epsilon_9}{2},
\end{align*}
again contradicting \prettyref{lem:D3}.
\end{proof}
The above analysis is performed at $\theta^0$. But by the continuity of problem data, we can extend the conditions to a small vicinity around $\theta^0$. We denote $\mathcal{B}(\theta^0,\epsilon_{13})$ as a closed ball around $\theta^0$ with a radius equal to $\epsilon_{13}$. We formalize this observation in the following lemma:
\begin{lemma}
\label{lem:D3Neighbor}
Taking \prettyref{ass:GMFCQ}, \ref{ass:Spatial}, and for $D_\theta^3$ stated in \prettyref{lem:D3}, we have:
\begin{align*}
&\left<D_\theta^3,\nabla_\theta\mathcal{O}(\theta)\right><-\frac{\epsilon_8}{2},
\end{align*}
 and one of the following condition holds each tuple of $\left<i,j,k\right>$:
 \begin{align}
&\dist(b_{ij}(t,\theta),o_k)>\frac{\epsilon_{12}}{2}\label{eq:CaseI}\\
&\left<D_\theta^3,\nabla_\theta\dist(b_{ij}(t,\theta),o_k)\right>\geq\frac{\epsilon_9}{4},\label{eq:CaseII}
 \end{align}
for any $\theta\in\mathcal{B}(\theta^0,\epsilon_{13})$.
\end{lemma}
\begin{proof}
Combining \prettyref{lem:D3}, \ref{lem:category}, and the continuity of problem data.
\end{proof}
\prettyref{lem:D3Neighbor} allows us to quantify the gradient norm of $\mathcal{E}(\theta,\mu)$ (we write $\mu$ as an additional parameter of $\mathcal{E}$ for convenience). The gradient norm should tend to zero as $k\to\infty$. However, \prettyref{lem:D3Neighbor} would bound it away from zero, leading to a contradiction.
\begin{lemma}
\label{lem:gradientBoundWay}
Taking \prettyref{ass:GMFCQ}, \ref{ass:Spatial}, \ref{ass:Bounded}, suppose $\{\mu^k\}$ is a null sequence, and for $D_\theta^3$ stated in \prettyref{lem:D3}, there exists a $D_\theta^3$ and sufficiently large $k$ such that for any $\theta\in\mathcal{B}(\theta^0,\epsilon_{14})$:
\begin{align*}
\left<D_\theta^3,\nabla_\theta\mathcal{E}(\theta,\mu^k)\right><-\epsilon_{15},
\end{align*}
for some positive constant $\epsilon_{14}$.
\end{lemma}
\begin{proof}
The gradient consists of three sub-terms:
\begin{align*}
&\nabla_\theta\mathcal{E}(\theta,\mu^k)=
\nabla_\theta\mathcal{O}(\theta)+
\nabla_\theta\bar{\mathcal{P}}_1(\theta,\mu^k)+
\nabla_\theta\bar{\mathcal{P}}_2(\theta,\mu^k)\\
&\bar{\mathcal{P}}_1(\theta,\mu^k)\triangleq\mu^k\sum_{ijkl}^{I}(T_1^l-T_0^l)\mathcal{P}_{ijkl}\\
&\bar{\mathcal{P}}_2(\theta,\mu^k)\triangleq\mu^k\sum_{ijkl}^{II}(T_1^l-T_0^l)\mathcal{P}_{ijkl}.
\end{align*}
Here we use $\sum_{ijkl}^{I}$ to denote a summation over terms $\mathcal{P}_{ijkl}$ where \prettyref{eq:CaseI} holds and $\sum_{ijkl}^{II}$ denotes a summation where \prettyref{eq:CaseI} does ont hold but \prettyref{eq:CaseII} holds. From \prettyref{lem:D3Neighbor}, we have the following holds for the first two terms:
\begin{align*}
\left<D_\theta^3,\nabla_\theta\mathcal{O}(\theta)\right><-\frac{\epsilon_8}{2}\quad
\left<D_\theta^3,\nabla_\theta\bar{\mathcal{P}}_2(\theta,\mu^k)\right><0.
\end{align*}
The third term can be arbitrarily small for sufficiently large $k$ because:
\begin{align*}
&\left|\left<D_\theta^3,\nabla_\theta\bar{\mathcal{P}}_1(\theta,\mu^k)\right>\right|\\
\leq&\mu^k\sum_{ijkl}^{I}(T_1^l-T_0^l)
\left|\left<D_\theta^3,\FPP{\mathcal{P}(x)}{x}\Bigg|_{\epsilon_{12}/2}
\nabla_\theta\bar{\dist}\right>\right|\\
\leq&\mu^k T\left|\FPP{\mathcal{P}(x)}{x}\Bigg|_{\epsilon_{12}/2}\right|
\left|\left<D_\theta^3,\nabla_\theta\bar{\dist}\right>\right|.
\end{align*} 
Here we have used \prettyref{ass:Bounded} and denote $\nabla_\theta\bar{\dist}$ as the location on trajectory where $\left|\left<D_\theta^3,\nabla_\theta\bar{\dist}\right>\right|$ takes the maximum value. Since $\{\mu^k\}$ is a null sequence, the third term can be arbitrarily small and the lemma follows.
\end{proof}
We are now ready to prove our main result:
\begin{proof}[Proof of \prettyref{thm:optimality}]
Suppose otherwise, i.e. the first-order optimality condition fails, then the condition of \prettyref{lem:gradientBoundWay} holds, indicating that $\|\nabla_\theta\mathcal{E}(\theta^k,\mu^k)\|_\infty$ is bounded away from zero. But this contradicts the fact that $\{\epsilon_d^k\}$ is null.
\end{proof}
\end{document}